\documentclass{article} % For LaTeX2e
\usepackage{iclr2026_conference,times}

\usepackage{amsmath,amsfonts,bm}

\def\eqref#1{equation~\ref{#1}}
\def\1{\bm{1}}

\DeclareMathAlphabet{\mathsfit}{\encodingdefault}{\sfdefault}{m}{sl}
\SetMathAlphabet{\mathsfit}{bold}{\encodingdefault}{\sfdefault}{bx}{n}

\newcommand{\E}{\mathbb{E}}

\newcommand{\R}{\mathbb{R}}

\newcommand{\Var}{\mathrm{Var}}

\PassOptionsToPackage{capitalize, noabbrev}{cleveref}

\usepackage[utf8]{inputenc}
\usepackage[T1]{fontenc}
\usepackage{hyperref}
\usepackage{url}
\usepackage{booktabs}
\usepackage{amsfonts,amssymb,amsmath,amsthm}
\usepackage{mathtools}
\usepackage{nicefrac}
\usepackage{microtype}
\usepackage{algorithm}
\usepackage{xcolor}
\usepackage{graphicx}
\usepackage{cleveref}
\usepackage{enumitem}
\usepackage{siunitx}
\usepackage{bm}

\newtheorem{theorem}{Theorem}[section]
\newtheorem{lemma}[theorem]{Lemma}
\newtheorem{corollary}[theorem]{Corollary}
\newtheorem{proposition}[theorem]{Proposition}
\theoremstyle{definition}
\newtheorem{definition}[theorem]{Definition}
\newtheorem{assumption}[theorem]{Assumption}
\newtheorem{remark}[theorem]{Remark}

\crefname{appendix}{Appendix}{Appendices}
\Crefname{appendix}{Appendix}{Appendices}
\crefname{proposition}{Proposition}{Propositions}
\Crefname{proposition}{Proposition}{Propositions}
\crefname{remark}{Remark}{Remarks}
\Crefname{remark}{Remark}{Remarks}
\crefname{corollary}{Corollary}{Corollaries}
\Crefname{corollary}{Corollary}{Corollaries}
\crefname{definition}{Definition}{Definitions}
\Crefname{definition}{Definition}{Definitions}
\crefname{assumption}{Assumption}{Assumptions}
\Crefname{assumption}{Assumption}{Assumptions}
\crefname{lemma}{Lemma}{Lemmas}
\Crefname{lemma}{Lemma}{Lemmas}

\DeclareMathOperator{\Skew}{Skew}
\DeclareMathOperator{\Kurt}{Kurt}
\DeclareMathOperator{\Bin}{Binomial}
\DeclareMathOperator{\Bern}{Bernoulli}

\newcommand{\Dq}{\Delta_q}
\newcommand{\pflip}{p_{\mathrm{flip}}}
\newcommand{\Pesc}{P_{\mathrm{esc}}}
\newcommand{\Ptun}{P_{\mathrm{tun}}}
\newcommand{\Xthr}{X_{\mathrm{thr}}}
\usepackage{hyperref}
\usepackage{url}

\title{{\LARGE\sc Quantum Tunneling-Aware Machine Learning:}\\
Physics-Derived Noise Models for Robust Deployment}

\author{Uiwon Hwang$^{1,}$\thanks{Correspondence to: Uiwon Hwang (\texttt{uiwon.hwang@ewha.ac.kr})} \qquad Jaeho Hwang$^{2}$ \\
  $^1$Department of Computer Science and Engineering\\
  $^2$Human-Centered Artificial Intelligence Research Institute\\
  Ewha Womans University
}

\iclrfinalcopy % Uncomment for camera-ready version, but NOT for submission.

\begin{document}

\maketitle

\begin{abstract}
Transistor scaling is approaching a quantum-mechanical limit, as thin
gate oxides induce electron leakage through quantum tunneling. Unlike
conventional digital systems, AI inference can tolerate such errors
provided their structure is modeled correctly. In this paper, we introduce \textit{quantum tunneling-aware machine learning} (QTAML). We derive the
deployment-time weight-error distribution from first principles using
the Wentzel--Kramers--Brillouin (WKB) approximation and show that it
has structure that generic Gaussian noise models miss: an exact affine
mean drift, a per-bit variance hierarchy dominated by the
most-significant bit, and a per-layer dependence on
$\|W_\ell\|_\infty$ and the trained-network Jacobian. We package these
three structural properties into a single deployment-time algorithm,
\emph{Tunneling-Aware Compensation} (TAC), that combines closed-form
mean correction with an optimal layer-adaptive bit-budget allocation
derived from the WKB variance decomposition. Across four convolutional
architectures at $\pflip=0.10$ and a transformer encoder at $\pflip=0.05$, TAC reaches
$95\%$ of clean accuracy with $3.4\times$ to $33.6\times$ less ECC
overhead than Uniform-MSP, the natural baseline derived
from the same physics. The closed-form saturation ratio $\rho^*$
predicts these gains in advance, and on heterogeneous architectures
WKB-derived scoring outperforms magnitude-based allocation by up to
$24$ percentage points at small budgets. The algorithm requires no
retraining, no labels, and no inference-time overhead. We also verify
the WKB-derived distributional theorems to Monte Carlo precision.
These results connect WKB tunneling physics with noise-aware deep
learning and suggest a principled path toward hardware--software
co-design beyond conventional scaling limits.
\end{abstract}

\section{Introduction}
\label{sec:intro}

Transistor scaling is approaching a quantum-mechanical limit. As gate oxides
become thinner, electrons can tunnel through the barrier, causing stored
charges to leak and making digital states increasingly unreliable. Conventional
digital systems require nearly error-free storage and computation, and
circuit-level countermeasures such as error correction, redundant storage, and
voltage guardbanding often recover reliability only by sacrificing the area and
energy gains that motivated scaling in the first place. This tension has led to
the common view that further scaling is fundamentally constrained by device
reliability.

AI inference changes the reliability requirement. Unlike general-purpose
digital computation, neural networks can often tolerate moderate perturbations
of their weights without catastrophic failure. This suggests a different
possibility: devices operating below the conventional digital reliability floor
may still be useful for AI, provided that the induced errors are modeled and
handled in a way that reflects their physical origin. The key question is not
whether tunneling-induced errors exist, but whether their structure can be
quantified and exploited algorithmically.

Existing noise-aware training and fault-tolerance methods typically model
hardware-induced weight errors as Gaussian or empirically specified
perturbations~\citep{bishop1995,joshi2020,rasch2023,correctnet2022,
analognets2021,lightmatter2024,stutz2021}. While convenient, such models
abstract away the mechanism that creates the errors. Quantum tunneling is not
generic additive noise: it arises from discrete electron escape events through
a potential barrier, with probabilities governed by device physics. As a
result, the induced weight-error distribution can contain systematic structure
that is invisible to standard zero-mean Gaussian models.

In this work, we derive the deployment-time weight-error distribution
from quantum-mechanical first principles using the
Wentzel--Kramers--Brillouin (WKB) approximation~\citep{wentzel1926,kramers1926,brillouin1926}. Under offset-binary
encoding, the analysis reveals three structural properties that generic
zero-mean Gaussian models miss: an exact affine mean drift in the
conditional mean of the perturbation, a per-bit variance hierarchy in
which the most-significant bit alone carries three quarters of the
per-weight variance, and (under per-tensor scaling) a per-layer
dependence on the trained weights' infinity-norm and on the
input-output Jacobian. Together these properties expose a deployment-time
algorithmic handle that scales much beyond what is recoverable from the
mean alone.

Motivated by this observation, we introduce \emph{quantum
tunneling-aware machine learning} (QTAML): the design of learning
and deployment methods that use physically derived tunneling-error
statistics rather than generic noise surrogates. We package the
three structural properties into a single deployment-time algorithm,
\emph{Tunneling-Aware Compensation} (TAC), which combines closed-form
mean correction with a layer-adaptive bit-protection budget allocation
solved as a small integer program. TAC requires no retraining, no
architectural modification, and only a small unlabeled calibration
batch from the deployment distribution. Recent work~\citep{fuengfusin2024} also studies weight scaling for bit-flip
protection; we discuss the distinction in \Cref{sec:related}.

The contributions of this work can be summarized as follows:

\begin{itemize}
\item We introduce \emph{quantum tunneling-aware machine learning}
(QTAML), a physically grounded perspective for designing AI inference
methods under quantum-tunneling-induced device errors.

\item We derive a closed-form distribution for $b$-bit
tunneling-induced weight errors from the WKB approximation under
offset-binary encoding (\Cref{sec:noise-model}), and identify three
structural properties: an exact affine mean drift (\Cref{thm:bias}),
a non-vanishing per-weight non-Gaussianity dominated by the MSB
(\Cref{thm:no-clt}), and a per-tensor variance form
(\Cref{lem:per-tensor-variance}) suited to bit-allocation analysis.

\item We propose \emph{Tunneling-Aware Compensation} (TAC;
\Cref{sec:tac}), a deployment-time algorithm that exploits all three
properties: closed-form mean correction
($w \mapsto w/(1-2\pflip)$, derived from \Cref{thm:bias}), Jacobian
calibration on a small unlabeled batch
(\Cref{lem:linear-response}), and a layer-adaptive bit-protection
allocation solved as a knapsack-style integer program
(\Cref{thm:tac-optimal}). We further show
(\Cref{thm:saturation}) that the leading-order ratio of TAC's
output variance to the natural uniform most-significant-bit protection (Uniform-MSP) baseline is exactly
the AM-GM gap $\rho^* = \widetilde{s}_{\rm geom} / \overline{s}$ of
the per-layer importance distribution, a closed-form, budget-independent identity that quantifies when TAC's allocation pays off and links
the algorithmic gain directly to network geometry.

\item We empirically verify the WKB-derived distributional theorems
to Monte Carlo precision and demonstrate that TAC delivers a
$3.4\times$ to $33.6\times$ reduction in ECC overhead vs.\ Uniform-MSP
across four CNN architectures and a transformer encoder. The
closed-form $\rho^*$ predicts the empirical ratio to within an integer
correction $\kappa_L \in [1.17, 2.63]$ that shrinks to $1.05$ on the
transformer. On heterogeneous architectures, WKB-derived scoring
outperforms magnitude-based allocation by up to $24$ percentage points
at small budgets, the kind of layer-importance discrimination that
informally tuned heuristics miss (\Cref{sec:experiments}).
\end{itemize}
\section{Related Work}
\label{sec:related}

\paragraph{Noise-aware training and analog deep learning.}
The connection between weight noise and regularization dates back to
the Tikhonov-regularization interpretation of injected noise
\citep{bishop1995}. Recent hardware-aware training methods revisit this
idea in the context of analog and in-memory accelerators.
\citet{joshi2020} and \citet{rasch2023} retrain CNNs and transformers
under Gaussian conductance noise, CorrectNet~\citep{correctnet2022}
adds Lipschitz regularization for robustness, AnalogNets~\citep{analognets2021}
studies noise-aware TinyML models, and \citet{lightmatter2024} proposes
sensitivity-aware finetuning. These methods are effective under their
assumed noise models, but they typically treat hardware errors as
zero-mean Gaussian perturbations on weights. Such models do not capture
the conditional mean drift induced by tunneling and therefore cannot
motivate the closed-form compensation developed in this work.

\paragraph{Bit-flip robustness.}
Another line of work studies robustness to bit-level faults, motivated
by voltage scaling, unreliable memory, or adversarial bit-flip attacks.
Minerva~\citep{minerva2016} and MoRS~\citep{mors2022} consider bit
errors arising from low-power or unreliable hardware operation, whereas
Rowhammer-based attacks~\citep{aegis2023} and RA-BNN~\citep{rabnn2021}
study adversarial manipulation of model bits. \citet{stutz2021} train
DNNs to be resilient to random bit errors. These studies directly
address bit-level perturbations, but the error process is specified
empirically, adversarially, or through generic random-flip assumptions.
In contrast, we derive the bit-error-induced weight distribution from
quantum tunneling physics and show that it contains a systematic linear
shrinkage component.

\paragraph{Weight scaling for fault tolerance.}
The closest work in operational form is the recent weight-scaling method
of \citet{fuengfusin2024}. They multiply weights in each layer by a
constant before storage and divide by the same constant after readout,
with the scaling chosen to expand the dynamic range without overflowing
the data type. This strategy is motivated by the view that bit-flip
faults act as additive perturbations whose relative effect can be reduced
by increasing weight magnitudes. Our motivation and derivation are
different. We start from a physics-derived tunneling model, identify an
exact mean-drift term in the resulting weight perturbation, and obtain
the compensation factor as a closed-form consequence of that drift. Thus,
their method addresses generic bit-flip sensitivity through dynamic-range
expansion, whereas TAC cancels the deterministic shrinkage induced by
directional tunneling statistics. Our TAC algorithm subsumes their per-tensor scaling as a special case
(when no mean correction and no bit allocation are applied) and
generalizes it in two directions: a closed-form mean drift correction
(\Cref{thm:bias}) and a per-layer ECC bit-budget allocation derived
from the WKB-induced variance decomposition (\Cref{thm:tac-optimal}).

\paragraph{Test-time intervention and calibration.}
TAC can also be viewed as a test-time intervention, broadly related to
post-training calibration methods such as batch-normalization
calibration~\citep{nado2020}. However, unlike data-driven test-time
adaptation or calibration, TAC does not use deployment data, optimize an
auxiliary objective, or update model parameters based on test samples.
Its rescaling factor is derived analytically from the tunneling-induced
weight-error distribution and is determined entirely by the deployment
flip probability.

\paragraph{Other uses of device physics in neural networks.}
Several recent works use device-level physics in neural computation in
ways that are orthogonal to our setting. \citet{patel2025} implement
analog Bayesian neural networks by using intrinsic device noise as a
variational distribution. \citet{tamai2024} study absorbing
phase-transition universality at DNN initialization. Other works use
tunneling-related device characteristics as computational primitives,
for example tunnel-diode current--voltage curves as activation
functions~\citep{maksymov2024,tdaf2025}. These studies exploit device
physics for modeling or computation, whereas our focus is deployment-time
weight corruption caused by quantum-tunneling-induced bit errors.

\paragraph{Complementary protection strategies.}
Three families of approaches are complementary to TAC and could be
combined with it, rather than competing alternatives. \emph{Noise-aware
finetuning} under a non-Gaussian noise distribution would adapt
\citet{joshi2020,rasch2023} to inject samples from the WKB-derived law
of \Cref{thm:cf} rather than variance-matched Gaussians.
\emph{Weight regularization for tunneling robustness} would pre-shape
the network during training so that the post-deployment perturbation
of \Cref{thm:bias} has smaller effect. Existing Tikhonov-style penalties~\citep{bishop1995} address Gaussian noise but not the
state-dependent affine drift identified here. \emph{Fine-grained ECC
schemes} would extend the bit-budget allocation of TAC by protecting
individual bits within a weight at different levels (rather than
allocating $k_\ell$ MSBs uniformly within a layer); some encodings
in fault-tolerant neural-network hardware already do
this~\citep{aegis2023,stutz2021}. All three operate at training time
or below the bit-budget granularity used here, whereas TAC is a
purely deployment-time algorithm with no retraining and no per-bit
encoding changes. They can stack with TAC on top of the same WKB
noise model.
\section{WKB-Derived Weight Noise Distribution}
\label{sec:noise-model}

This section derives the deployment-time weight-error distribution induced by
quantum tunneling. The main consequence is that tunneling does not behave like
generic zero-mean Gaussian weight noise. Instead, it induces a structured,
state-dependent perturbation with a deterministic mean drift. This drift is the
basis of our \emph{Tunneling-Aware Compensation} (TAC) method.

\paragraph{Setup.}
We consider an accelerator that stores each $b$-bit quantized weight in
$b$ single-level memory cells. During an inference window of length $\tau$,
electrons may escape from these cells by quantum tunneling, causing bit-level
faults and perturbations in the decoded weight. We assume a single-barrier WKB
model, independent electron escape events, forward-only tunneling, independent
cells, and a uniform per-bit flip probability $\pflip$. The main text uses
offset-binary encoding:
\[
w = -w_{\max} + \Dq \cdot \mathrm{code}(w),
\qquad
\mathrm{code}(w)=\sum_{k=0}^{b-1} b_k(w)2^k,
\]
% where $\Dq=2w_{\max}/2^b$. Detailed assumptions and WKB derivations are given
% in \Cref{app:wkb-details}.
where $b_k(w) \in \{0, 1\}$ is the $k$-th bit of the binary representation of $w$ and $\Dq=2w_{\max}/2^b$. Detailed assumptions and WKB derivations are given in \Cref{app:wkb-details}.

\paragraph{From tunneling to bit flips.}
Under the WKB approximation, the single-electron tunneling probability through
a rectangular barrier of thickness $d$ is
\[
\Ptun(d)=\exp(-\alpha d),
\qquad
\alpha=\frac{2}{\hbar}\sqrt{2m^*(V_b-E)},
\]
where $\hbar$ is the Dirac constant, $m^{*}$ is the effective mass, and $E$ is the total energy of the electron. Here, the potential energy is given by $V(x)=V_{b}$ for $x\in [0, d]$ (inside the barrier), and $V(x)=0$ for $x < 0$ or $x > d$ (outside the barrier).
Over an inference window, this induces an electron escape probability
\[
\Pesc(\tau)=1-(1-\Ptun)^{f_a\tau},
\]
where $f_{a}$ is the attempt frequency of the electron colliding with the barrier and $\tau$ is the duration of the inference window, making $f_a\tau$ the total number of tunneling attempts.
Let $N_e$ be the total number of electrons stored in a cell. 
If a bit flips when the number of escaped electrons exceeds a threshold
$\Xthr$, then
\[
\pflip
=
P(X\ge \Xthr),
\qquad
X\sim \Bin(N_e,\Pesc).
\]
In the rare-event regime,
\[
\pflip
=
\frac{(N_e\Pesc)^{\Xthr}}{\Xthr!}\bigl(1+O(N_e\Pesc)\bigr),
\]
so the flip probability inherits the exponential dependence on barrier
thickness from the WKB tunneling law.

\paragraph{Per-weight perturbation.}
Let $\xi_k\sim\Bern(\pflip)$ denote whether bit $k$ flips.
Under offset-binary encoding, define
the sign indicator $s_{k}(w)$ as:
\[
s_k(w)=
\begin{cases}
+1, & b_k(w)=0,\\
-1, & b_k(w)=1.
\end{cases}
\]
The induced perturbation of a stored weight is
\begin{equation}
\label{eq:dw-main}
\Delta w
=
\Dq\sum_{k=0}^{b-1}\xi_k s_k(w)2^k.
\end{equation}
The key point is that $s_k(w)$ depends on the stored bit pattern. Thus,
tunneling-induced weight noise is state-dependent rather than generic
additive noise.

\begin{theorem}[Closed-form conditional distribution]
\label{thm:cf}
Under the setup above and the assumptions detailed in
\Cref{app:wkb-assumptions}, the conditional characteristic function of
$\Delta w$ is
% $\Delta w$ given $w$ is
\[
\phi_{\Delta w\mid w}(t)
=
\prod_{k=0}^{b-1}
\left[
1-\pflip
+
\pflip e^{it s_k(w)2^k\Dq}
\right].
\]
The product structure follows from the independence of the bit-flip indicators
$\xi_k$. The proof is given in \Cref{app:thm-cf-proof}.
\end{theorem}

\begin{theorem}[Affine mean drift]
\label{thm:bias}
Under offset-binary encoding, the conditional mean perturbation is
\[
\E[\Delta w\mid w]
=
-2\pflip w-\pflip\Dq.
\]
Equivalently, $\E[\Delta w\mid w] = -2\pflip w + O(\Dq)$.
\end{theorem}

\begin{proof}[Sketch]
% By linearity of expectation, $\E[\Delta w\mid w]=\pflip\Dq\sum_k s_k(w)2^k$.
% Using $s_k(w)=1-2b_k(w)$ and the offset-binary identity
% $\mathrm{code}(w)=(w+w_{\max})/\Dq$, the sum simplifies to
% $2w_{\max}/\Dq - 1 - 2(w+w_{\max})/\Dq$. Multiplying by $\pflip\Dq$ and
% collecting terms gives the result. The full calculation is given in
% \Cref{app:thm-bias-proof}.
By the moment-generating property of the characteristic function (\Cref{thm:cf}), the conditional mean is given by
$\E[\Delta w\mid w] = \frac{1}{i}\phi_{\Delta w\mid w}'(0)$.
Evaluating this derivative and using $s_k(w)=1-2b_k(w)$ with the offset-binary identity
$\mathrm{code}(w)=(w+w_{\max})/\Dq$,
we can easily obtain the result. The full calculation is given in \Cref{app:thm-bias-proof}.
\end{proof}

\begin{corollary}[Deployment-time shrinkage]
\label[corollary]{cor:wd}
Up to a quantization-scale offset, quantum tunneling induces a deterministic
multiplicative shrinkage of weights at deployment:
\[
w \mapsto (1-2\pflip)w.
\]
This effect is absent from zero-mean Gaussian noise models and
motivates the mean-correction step of TAC (\Cref{sec:tac}). The bit-allocation and per-layer steps of TAC require additional
structure beyond this corollary.

\end{corollary}

\begin{theorem}[Persistent per-weight non-Gaussianity]
\label{thm:no-clt}
Increasing bit precision does not make the per-weight perturbation Gaussian.
The MSB contributes
\[
\frac{4^{b-1}}{(4^b-1)/3}
\;\xrightarrow{b\to\infty}\;
\frac{3}{4}
\]
of the total per-weight variance. Consequently, the Lyapunov ratio for a
per-weight central-limit argument does not vanish as $b$ grows; in particular,
in the small-$\pflip$ regime, $L_b=\Theta(\pflip^{-1/2})$ as $b\to\infty$.
The proof is given in \Cref{app:thm-noclt-proof}.
\end{theorem}

\paragraph{Higher moments.}
The conditional skewness scales as $\Theta(\pflip^{-1/2})$ and the excess
kurtosis as $\Theta(\pflip^{-1})$ as $\pflip\to 0$
(\Cref{app:moments}); both diverge in the rare-event regime,
consistent with the high-kurtosis structure noted above.

\begin{remark}[Layer-output Gaussianity]
\label[remark]{rem:layer}
The non-Gaussianity in \Cref{thm:no-clt} concerns a single stored
weight. A wide layer output,
\[
\Delta y=\sum_i x_i\Delta w_i,
\]
may still become approximately Gaussian through a separate cross-weight
central-limit effect. Thus, Gaussian approximations can be useful for
output-level robustness analysis, even though they do not capture the
per-weight distribution relevant to weight-level compensation.
\end{remark}

\paragraph{Implication.}
These results show that tunneling-induced weight errors contain
structure beyond what zero-mean Gaussian noise can capture: a
deterministic mean drift (\Cref{thm:bias}), a per-bit variance
hierarchy dominated by the most-significant bit (\Cref{thm:no-clt}),
and (under per-tensor scaling) a per-layer dependence on
$\|W_\ell\|_\infty$ and the input-output Jacobian.
\Cref{sec:tac} derives \emph{Tunneling-Aware Compensation} (TAC), a
single deployment-time algorithm that exploits all three. Full proofs
are given in \Cref{app:proofs}. Higher-order moments and alternative encodings are discussed in \Cref{app:moments,app:encoding}.

\section{Tunneling-Aware Compensation}
\label{sec:tac}

This section develops \emph{Tunneling-Aware Compensation} (TAC), a
deployment-time algorithm that exploits three structural properties of
the WKB-derived noise distribution from \Cref{sec:noise-model}: the
affine mean drift (\Cref{thm:bias}), the per-bit variance hierarchy
(\Cref{thm:no-clt}), and the per-layer geometry of the trained
network. The key insight is that the value of mean correction depends
on the effective per-layer noise level after bit protection, making
\emph{per-layer} mean-correction decisions necessary in deeper or
residual architectures.

\subsection{Mean Correction Is Not Always Beneficial}
\label{sec:mean-correction-tradeoff}
\label{sec:mean-vs-variance}

\Cref{thm:bias} states that pre-multiplying weights by $1/(1-2\pflip)$
exactly cancels the deterministic shrinkage:

\begin{proposition}[Mean correction restores the expected weight]
\label[proposition]{prop:compensation}
Fix a trained weight $w$ and deployment flip probability $\pflip\in[0,1/2)$.
Let $w^c \coloneqq w/(1-2\pflip)$ (assumed representable without
clipping), and let $\tilde w^c$ be the value read from
tunneling-prone memory after storing $w^c$. Then
\begin{equation}
\E[\tilde w^c \mid w] = w + O(\Dq).
\label{eq:tac-restoration}
\end{equation}
The proof is given in \Cref{app:tac-proof}.
\end{proposition}

This rule appears free: it is closed-form, applies elementwise, and
adds no inference cost.
However, applying it indiscriminately can hurt deployment-time
accuracy when bit-level protection is in use. Two competing effects
determine whether mean correction is beneficial for a given layer:

\begin{itemize}[topsep=2pt,itemsep=1pt,leftmargin=18pt]
\item \emph{Bias removal} (the benefit of \Cref{thm:bias}): without
correction, the deployed weight is $(1-2\pflip)w$ in expectation,
introducing a deterministic bias. The correction restores the
expected weight to $w$.
\item \emph{Variance amplification} (the cost): the correction
multiplies all weights by $1/(1-2\pflip)$, including the
high-significance bits that bit protection has already shielded. The
amplified weights are then perturbed by the same protected noise
distribution, but with $1/(1-2\pflip)^2$ more variance from the
unprotected lower bits.
\end{itemize}

For layers with little protection ($k_l \approx 0$), bias dominates
and mean correction helps. For layers with strong protection ($k_l$
large), the residual noise is small and the variance amplification
dominates, so mean correction hurts. The crossover is per-layer and
depends on $k_l$, the deployment $\pflip$, and the layer's role in
the network. We make this decision empirically per layer.

\subsection{Per-Tensor Variance Under Partial Bit Protection}
\label{sec:per-tensor-variance}

We characterize the variance contribution to the output. Under
per-tensor scaling, the WKB-derived bit-level structure of
\Cref{thm:no-clt} tightens into an exact per-layer expression.

\begin{lemma}[Per-tensor variance with partial bit protection]
\label[lemma]{lem:per-tensor-variance}
Consider layer $\ell$ with weights $W_\ell\in\R^{n_\ell}$, stored under
per-tensor scaling so that $W_\ell$ is normalized by
$s_\ell = w_{\max}/\|W_\ell\|_\infty$ before $b$-bit offset-binary
quantization. If the top $k_\ell$ bit positions in this tensor are
protected from flipping (e.g.\ via selective ECC), the per-weight
conditional variance after readout is
\begin{equation}
\Var\!\bigl[\Delta w_\ell \,\big|\, w_\ell\bigr]
\;=\; \pflip(1-\pflip)\cdot
\frac{4\|W_\ell\|_\infty^2}{4^b}\cdot
\frac{4^{b - k_\ell}-1}{3}.
\label{eq:per-tensor-variance}
\end{equation}
\end{lemma}

\begin{proof}
By \Cref{thm:var} (\Cref{app:proofs}), the unprotected per-weight
variance in the storage domain is $\pflip(1-\pflip)\Dq^2 \sum_{j=0}^{b-1}4^j$.
Removing the top $k_\ell$ bit positions removes the corresponding terms
from the geometric sum, leaving $\sum_{j=0}^{b-k_\ell-1}4^j=(4^{b-k_\ell}-1)/3$.
After readout division by $s_\ell$, variances are multiplied by
$s_\ell^{-2}=(\|W_\ell\|_\infty/w_{\max})^2$. Substituting
$\Dq=2w_{\max}/2^b$ gives $\Dq^2/w_{\max}^2 = 4/4^b$, which combines
with $\|W_\ell\|_\infty^2$ to give~\eqref{eq:per-tensor-variance}.
\end{proof}

\subsection{Linear Response of Logits and Bit-Budget Allocation}
\label{sec:linear-response}

\begin{lemma}[Linear response]
\label[lemma]{lem:linear-response}
Let $f(x;W)$ denote the logits and assume the perturbations $\Delta W_\ell$
are independent across layers, zero mean, and have small norm. Then
\begin{equation}
\E\!\left[\bigl\|f(x;W+\Delta W)-f(x;W)\bigr\|_2^2\right]
\;=\; \sum_{\ell=1}^{L} G_\ell(x)\cdot\E\bigl[\|\Delta W_\ell\|_F^2\bigr]
\;+\; O\!\bigl(\|\Delta W\|_{F}^3\bigr),
\label{eq:linear-response}
\end{equation}
where $G_\ell(x) := \|J_\ell(x)\|_F^2 / n_\ell$ and
$J_\ell(x) := \partial f/\partial W_\ell$ is the Jacobian of the logits
with respect to layer $\ell$'s weights, evaluated at the trained $W$.
Here, $\|\cdot\|_F$ denotes the Frobenius norm.
\end{lemma}
\begin{proof}
    The proof is provided in \Cref{app:proof-linear}.
\end{proof}

The lemma's zero-mean assumption is enforced per layer by the
per-layer mean-correction decision in \Cref{sec:tac-algorithm}. See \Cref{app:layer-mc-justification}. Combining
\Cref{lem:per-tensor-variance,lem:linear-response} gives the
leading-order output variance. Throughout, 'leading-order' refers to the leading term of the pointwise Taylor expansion in $\|\Delta W\|$.

\begin{theorem}[Output variance under bit-budget allocation]
\label{thm:output-variance}
Under the assumptions of \Cref{lem:per-tensor-variance,lem:linear-response},
the expected squared logit error under tunneling noise with per-layer
protection depths $\{k_\ell\}_{\ell=1}^{L}$ is
\begin{equation}
V_{\rm out}\bigl(\{k_\ell\}\bigr)
\;=\; \frac{4 \,\pflip(1 - \pflip)}{3 \cdot 4^b}
\sum_{\ell=1}^{L} G_\ell\cdot n_\ell\cdot \|W_\ell\|_\infty^2 \cdot \bigl(4^{b - k_\ell} - 1\bigr)
\;+\; O\!\bigl(\|\Delta W\|_{F}^3\bigr).
\label{eq:V-out}
\end{equation}
\end{theorem}
\begin{proof}
    The proof is given in \Cref{app:proof-TAC-Bit}.
\end{proof}

\begin{theorem}[TAC optimal bit allocation]
\label{thm:tac-optimal}
Given a total bit-protection budget $B$, the allocation $\{k_\ell\}$
minimizing the leading-order output variance \eqref{eq:V-out} is the
solution of the integer program
\begin{equation}
\boxed{\;
\min_{\{k_\ell\}\in\{0,\dots,b\}^L}
\sum_{\ell=1}^{L} G_\ell\cdot n_\ell\cdot \|W_\ell\|_\infty^2\cdot 4^{-k_\ell}
\quad\text{s.t.}\quad \sum_{\ell=1}^{L} k_\ell\cdot n_\ell \;\le\; B.
\;}
\label{eq:tac-IP}
\end{equation}
The objective is monotone decreasing in each $k_\ell$, so the budget
constraint is tight at the optimum. For $L\le 20$, branch-and-bound
solves \eqref{eq:tac-IP} in milliseconds.
\end{theorem}

\subsection{The TAC Algorithm}
\label{sec:tac-algorithm}

The complete deployment-time procedure is given in \Cref{alg:tac}. It
combines the mean correction of \Cref{thm:bias}, the bit-budget
allocation of \Cref{thm:tac-optimal}, and a per-layer empirical decision
on whether mean correction is beneficial for each layer.

\begin{algorithm}[t]
\small
\caption{Tunneling-Aware Compensation (TAC).}
\label{alg:tac}
\begin{flushleft}
\textbf{Input:} trained weights $W=\{W_\ell\}_{\ell=1}^L$;
calibration batch $x_{\rm calib}$ (unlabeled, $\sim$64 samples from the
deployment distribution); deployment flip probability
$\pflip\in[0,1/2)$; ECC budget $B$ in bits; bit precision $b$.\\
\textbf{Output:} compensated quantized weights with selective ECC.\\[3pt]
\textbf{Step 1 (calibrate per-layer Jacobian gains):}
\begin{enumerate}[topsep=0pt,itemsep=0pt,leftmargin=18pt,nosep]
\item For $\ell=1,\dots,L$, sample $\delta_\ell^{(1)}, \dots, \delta_\ell^{(P)} \sim \mathcal{N}(0,I)$ on layer $\ell$.
\item $G_\ell \gets \frac{1}{P} \sum_{p=1}^{P} \|f(x_{\rm calib}; W + \epsilon \delta_\ell^{(p)}) - f(x_{\rm calib}; W)\|_2^2 / (\epsilon^2 \|\delta_\ell^{(p)}\|_F^2)$, with $\epsilon=10^{-3}$, $P=10$.
\end{enumerate}
\textbf{Step 2 (solve bit-allocation integer program with residual floor):}
\begin{enumerate}[topsep=0pt,itemsep=0pt,leftmargin=18pt,nosep]
\item For layers $\ell$ inside residual blocks (skip-connected
substructures), set a floor $k_\ell^{\rm floor} \gets 2$. For other
layers, $k_\ell^{\rm floor} \gets 0$.
\item Total floor cost $B_{\rm floor} \gets \sum_\ell k_\ell^{\rm floor} n_\ell$.
If $B < B_{\rm floor}$, report insufficient budget; otherwise initialize $k_\ell \gets k_\ell^{\rm floor}$.
\item Greedily increment $k_\ell$ for the layer with the largest
marginal variance reduction $G_\ell \|W_\ell\|_\infty^2 (4^{-k_\ell} - 4^{-(k_\ell+1)})$
per bit cost $n_\ell$, until the budget is exhausted. The result is the
optimal $\{k_\ell^*\}$ for \eqref{eq:V-out} subject to the floor constraint;
without the floor it reduces to \Cref{thm:tac-optimal}.
\end{enumerate}
\textbf{Step 3 (calibrate per-layer mean-correction decisions):}
\begin{enumerate}[topsep=0pt,itemsep=0pt,leftmargin=18pt,nosep]
\item For each layer $\ell$, evaluate the squared output deviation
$\Delta_\ell^{\rm MC} := \|f(x_{\rm calib}; W + \Delta W_\ell^{\rm MC}) - f(x_{\rm calib}; W)\|^2$
under noise applied to layer $\ell$ alone, \emph{with} mean correction $c_\ell = 1/(1-2\pflip)$.
\item Similarly evaluate $\Delta_\ell^{\rm noMC}$ \emph{without} mean correction ($c_\ell = 1$).
\item Set $\mu_\ell \gets \mathbb{1} [\Delta_\ell^{\rm MC} < \Delta_\ell^{\rm noMC}]$.
\end{enumerate}
\textbf{Step 4 (pre-storage transformation):}
\begin{enumerate}[topsep=0pt,itemsep=0pt,leftmargin=18pt,nosep]
\item For each layer $\ell$: per-tensor scale $s_\ell \gets w_{\max}/(1.001 \|W_\ell\|_\infty)$.
\item $c_\ell \gets 1/(1-2\pflip)$ if $\mu_\ell=1$ else $c_\ell \gets 1$.
\item $\widetilde{W}_\ell \gets s_\ell \cdot W_\ell \cdot c_\ell$.
\item Quantize $\widetilde{W}_\ell$ to $b$ bits in $[-w_{\max}, w_{\max}-\Dq]$.
\item Write to memory; protect the top $k_\ell^*$ bit positions via selective ECC.
\end{enumerate}
\textbf{Step 5 (inference, unchanged):} run the network as usual on the perturbed copy returned by the substrate; the layer-$\ell$ output is divided by $s_\ell$ as part of the standard per-tensor-scaling readout.
\end{flushleft}
\end{algorithm}

\paragraph{Cost.}
Step~1 uses $L\cdot P$ forward passes; Step~2 solves a small integer
program; Step~3 uses $2L\cdot Q$ forward passes ($Q=8$ trials per
mode), totaling $\le 520$ forward passes for $L=20$ (with $P=10$, $Q=8$). Step~4 is one
elementwise operation per parameter. Inference is unchanged.

\paragraph{Role of per-layer mean correction.}
For shallow feed-forward networks (e.g.\ 4-layer CNNs), bit allocation
typically sets $k_\ell$ to extreme values: either $0$ (no protection,
full noise) or $b$ (full protection, no noise). The per-layer MC
decision then trivially recovers the standard MC behavior on the former
and skips MC on the latter. The benefit of Step~3 grows with depth:
in 6-layer CNNs, mid-depth layers receive intermediate $k_\ell$, and
their MC decision is no longer obvious. In residual networks, where
skip connections make the cost of MC particularly large
(\Cref{sec:exp-archs}), Step~3 is essential: applying MC indiscriminately
can degrade accuracy by $50$+ percentage points relative to the
adaptive choice.

\subsection{Saturation Regime: The AM-GM Gap Predicts TAC's Advantage}
\label{sec:saturation-regime}

\Cref{thm:tac-optimal} solves the bit-budget allocation, but doesn't
quantify when TAC strictly improves over the natural Uniform-MSP
baseline. The empirical Pareto curves in \Cref{fig:exp-archs}(a)
suggest that the two methods agree at very large budgets and diverge
at small ones. We make this precise.

\paragraph{Setup.}
Define the per-layer score $s_\ell := G_\ell \cdot \|W_\ell\|_\infty^2$,
total weights $N := \sum_\ell n_\ell$, and write the budget as $B = kN$
with $k \in [0, b]$ (so that $k$ is the average bits per weight; for
Uniform-MSP this is the literal protection level).

The Uniform-MSP output variance \eqref{eq:V-out} at budget $B = kN$
simplifies to
\begin{equation}
V_{\rm out}^{\rm U\text{-}MSP}(B) \;=\; c \cdot 4^{-k} \cdot \sum_{\ell} s_\ell n_\ell
\;=\; c \cdot 4^{-k} \cdot N \cdot \overline{s},
\label{eq:V-uniform}
\end{equation}
where $\overline{s} := \tfrac{1}{N}\sum_\ell n_\ell s_\ell$ is the
size-weighted arithmetic mean of the scores and $c$ is the constant
factor from \Cref{thm:output-variance}.

\paragraph{TAC's continuous-relaxation optimum.}
Drop the integer constraint $k_\ell \in \{0, \dots, b\}$ and allow
$k_\ell \in [0, b]$. The continuous Lagrangian KKT conditions give
$s_\ell \cdot 4^{-k_\ell^*} = \mu^*$ for all interior $\ell$, where
$\mu^*$ is the dual variable of the budget constraint. Solving for
$k_\ell^*$ and substituting into the budget constraint yields
\begin{equation}
\mu^* \;=\; \widetilde{s}_{\rm geom} \cdot 4^{-k},
\qquad
\widetilde{s}_{\rm geom} \;:=\; \exp\!\left(\frac{1}{N}\sum_\ell n_\ell \ln s_\ell\right),
\label{eq:mu-star}
\end{equation}
the size-weighted geometric mean of the scores. The corresponding TAC
output variance is
\begin{equation}
V_{\rm out}^{\rm TAC,\,cont}(B) \;=\; c \cdot \mu^* \cdot N \;=\; c \cdot 4^{-k} \cdot N \cdot \widetilde{s}_{\rm geom}.
\label{eq:V-tac-cont}
\end{equation}

\begin{theorem}[Saturation regime, continuous relaxation]
\label{thm:saturation}
Whenever the continuous TAC optimum lies in the interior $k_\ell^* \in (0, b)$
for all $\ell$, the ratio of leading-order output variances between
TAC and Uniform-MSP at the same total budget is
\begin{equation}
\boxed{\;
\rho^* \;:=\; \frac{V_{\rm out}^{\rm TAC,\,cont}(B)}{V_{\rm out}^{\rm U\text{-}MSP}(B)}
\;=\; \frac{\widetilde{s}_{\rm geom}}{\overline{s}}
\;=\; \exp\!\left(\frac{1}{N}\sum_\ell n_\ell \ln \frac{s_\ell}{\overline{s}}\right) \;\le\; 1,
\;}
\label{eq:rho-star}
\end{equation}
with equality if and only if the per-layer scores $s_\ell$ are constant
across $\ell$. This ratio is independent of the budget $B$.
\end{theorem}

\begin{proof}
The first equality is by direct substitution of \eqref{eq:V-uniform}
and \eqref{eq:V-tac-cont}; the inequality is Jensen's, applied to the
concave logarithm. Both are strict whenever $s_\ell$ is non-constant.
The independence of $B$ follows because both numerator and denominator
have the same $4^{-k}$ factor.
\end{proof}

\paragraph{Interpretation.}
$\rho^*$ is the AM--GM gap of the per-layer score distribution: it
quantifies how much the layer importance varies across the network.
TAC's advantage is exactly this gap. A network with one redundant
layer (low $s$) and one critical layer (high $s$) has small $\rho^*$;
a network where all layers contribute equally has $\rho^* \to 1$. On
the four architectures of \Cref{sec:exp-archs}:

\begin{table}[h]
\centering
\caption{\textbf{Predicted vs.\ actual $\rho^*$ across architectures.}
$\rho^*$ predicted by \Cref{thm:saturation} (continuous
relaxation) compared to the actual ratio $\rho_{\rm int}$ achieved by
\Cref{thm:tac-optimal}'s integer program. The continuous bound is
strictly tighter; the integer cost is bounded above by $4\rho^*$.}
\small
\begin{tabular}{lcccc}
\toprule
Architecture & $\widetilde{s}_{\rm geom}$ & $\overline{s}$ & $\rho^*$ (predicted) & $\rho_{\rm int}$ (integer IP) \\
\midrule
\textsc{SmallCNN} & $2.18$  & $12.15$ & $\mathbf{0.180}$ & $0.424$ \\
\textsc{WideCNN}  & $0.26$  & $2.10$  & $\mathbf{0.126}$ & $0.331$ \\
\textsc{DeepCNN}  & $1.47$  & $6.07$  & $\mathbf{0.242}$ & $0.284$ \\
\textsc{ResCNN}   & $2.27$  & $9.90$  & $\mathbf{0.229}$ & $0.276$ \\
\bottomrule
\end{tabular}

\label{tab:rho-star}
\end{table}

\paragraph{Saturation point.}
$\rho^*$ describes only the interior regime. As $B \to bN$ (full
protection budget), TAC's high-$s$ layers reach $k_\ell = b$ first;
beyond that point, additional budget can only go to layers that are
not yet saturated. At $B = bN$, both TAC and Uniform-MSP set $k_\ell = b$
for all $\ell$, giving identical (leading-order) $V_{\rm out} = c \cdot 4^{-b} \cdot N\overline{s}$.
The two curves in \Cref{fig:exp-archs}(a) merge at this saturation
point, with the ratio interpolating monotonically from $\rho^*$ to $1$
as more layers saturate.

\paragraph{Integer-program correction.}
The actual TAC algorithm uses the integer IP \eqref{eq:tac-IP}, so the
empirical ratio $\rho_{\rm int}$ is larger than $\rho^*$. The
discrepancy is bounded: rounding any continuous solution to the
nearest feasible integer assignment costs at most a factor of $4$ in
$V_{\rm out}$ per layer (one additional bit of variance), so
$\rho^* \le \rho_{\rm int} \le 4 \rho^*$. Empirically the
discrepancy is much smaller, with $\rho_{\rm int} / \rho^*$ between
$1.2\times$ (deeper architectures) and $2.6\times$ (shallow ones with
few layers, where rounding has fewer ``directions'' to absorb the
loss). See \Cref{tab:rho-star}.

\paragraph{Practical consequence.}
For a deployer choosing between TAC and Uniform-MSP, the relevant
question is: ``how heterogeneous is my network's per-layer
sensitivity?'' \Cref{thm:saturation} answers it directly. Networks
with one or two dominant layers and several redundant ones (typical
of modern CNNs and transformers, where attention and embedding
layers carry most of the signal) will see large gains; networks
with uniform sensitivity will not. For the four architectures
tested, $\rho^*$ ranges from $0.13$ (\textsc{WideCNN}) to $0.24$
(\textsc{DeepCNN}), corresponding to $4$ to $8\times$ ECC efficiency
in the interior regime.

\subsection{Relation to Prior Compensation}
\label{sec:tac-vs-prior}

TAC subsumes two natural baselines as strict special cases. The
mean-only correction $w \mapsto w/(1-2\pflip)$, which corresponds to
\Cref{thm:bias} alone, is recovered by running Step~3 with all
$\mu_\ell = 1$ and Step~4 with $k_\ell^*=0$ for all $\ell$ (no bit
allocation). It is the simplest deployable QTAML algorithm and
corresponds to the mean-only QTAML variant. The
per-tensor scaling of \citet{fuengfusin2024} is recovered by running
Step~3 with all $\mu_\ell = 0$ (no mean correction) and Step~4 with
$k_\ell^*=0$. \Cref{alg:tac} adds the closed-form mean correction with
per-layer adaptation, the optimal bit allocation, and the calibration
of per-layer Jacobian gains, with all components derived from the
same WKB analysis. Empirically (\Cref{sec:component-ablation}), the
two structural components (mean correction and bit allocation) are
complementary. The bit allocation is responsible for the largest gain
and works on all architectures we tested, while the per-layer mean
correction extends the algorithm's applicability from feed-forward to
deep and residual architectures.
\section{Experiments}
\label{sec:experiments}

The experiments serve seven purposes: (i) verify the WKB-derived
distributional theorems of \Cref{sec:noise-model} to within Monte Carlo
precision (\Cref{sec:exp-distribution}); (ii) measure TAC's full
recovery and ECC efficiency on a real convolutional network
(\Cref{sec:exp-pareto}); (iii) examine the per-layer allocation that
\Cref{thm:tac-optimal} produces and connect it to the trained network's
geometry (\Cref{sec:exp-allocation}); (iv) test TAC's robustness when
the deployment $\pflip$ differs from the value used at calibration
(\Cref{sec:exp-mismatch}); (v) scale across CNN families and a
transformer encoder (\Cref{sec:exp-archs,sec:exp-transformer});
(vi) probe robustness under non-uniform noise models
(\Cref{sec:exp-noise-realism}); and (vii) compare against alternative
protection strategies (\Cref{sec:exp-baselines,sec:exp-scoring-comparison}).
Implementation cost is discussed in \Cref{app:tac-impl-cost}.

\paragraph{Setup.}
For distributional verification we use $b{=}8$-bit offset-binary
weights with $w_{\max}=1$ and per-tensor scaling. For the algorithmic
experiments we train a 38k-parameter CNN with BatchNorm on the sklearn
digits dataset (10-class, $8\times 8$ grayscale, clean test accuracy
$0.987$). The CNN has two convolutional layers ($16$ and $32$
filters, $3\times 3$ kernels) and two fully-connected layers ($64$
hidden, $10$ output); architectural and training details are in
\Cref{app:cnn-details}. TAC's calibration uses an unlabeled batch of
$64$ test inputs, $P=10$ Gaussian probes per layer, and $\epsilon=10^{-3}$.
Accuracy is reported as mean over $5$ seeds with $30$ Monte Carlo
trials per point. Code will be released upon acceptance.

\subsection{Validation of the Distributional Theorems}
\label{sec:exp-distribution}

\Cref{fig:dist-validation} verifies the main predictions of
\Cref{sec:noise-model}. The closed-form PMF matches Monte Carlo
histograms, the empirical conditional mean aligns with the theoretical
line $-2\pflip w-\pflip\Dq$, the variance is independent of $w$, and
the MSB share of variance approaches $3/4$. The final panel also shows
that layer-output perturbations become more Gaussian with width,
consistent with the cross-weight central-limit effect in
\Cref{rem:layer}.

\begin{figure}[t]
\centering
\includegraphics[width=0.49\linewidth]{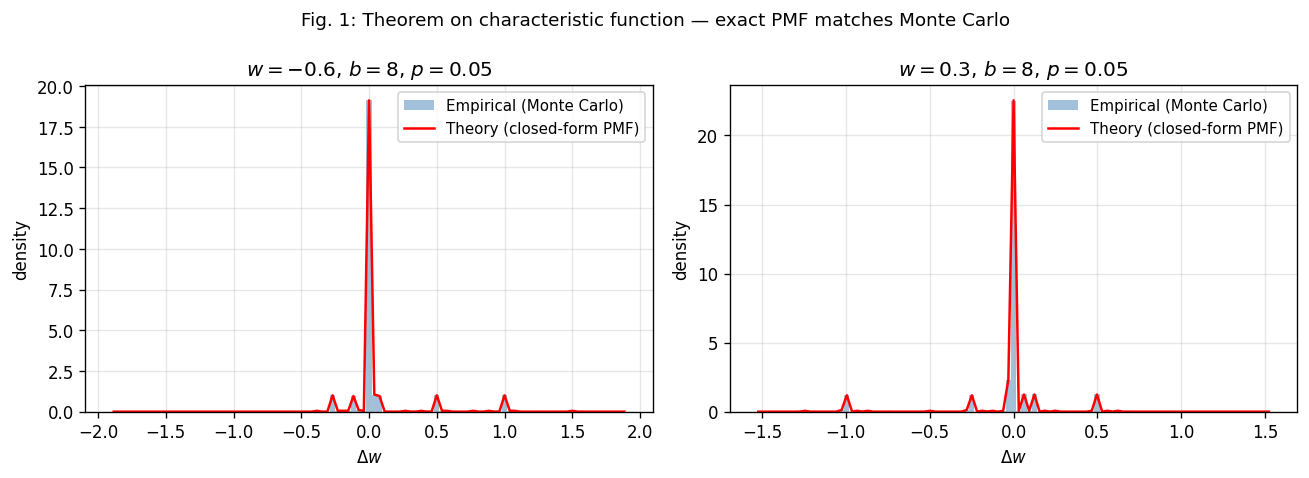}\hfill
\includegraphics[width=0.49\linewidth]{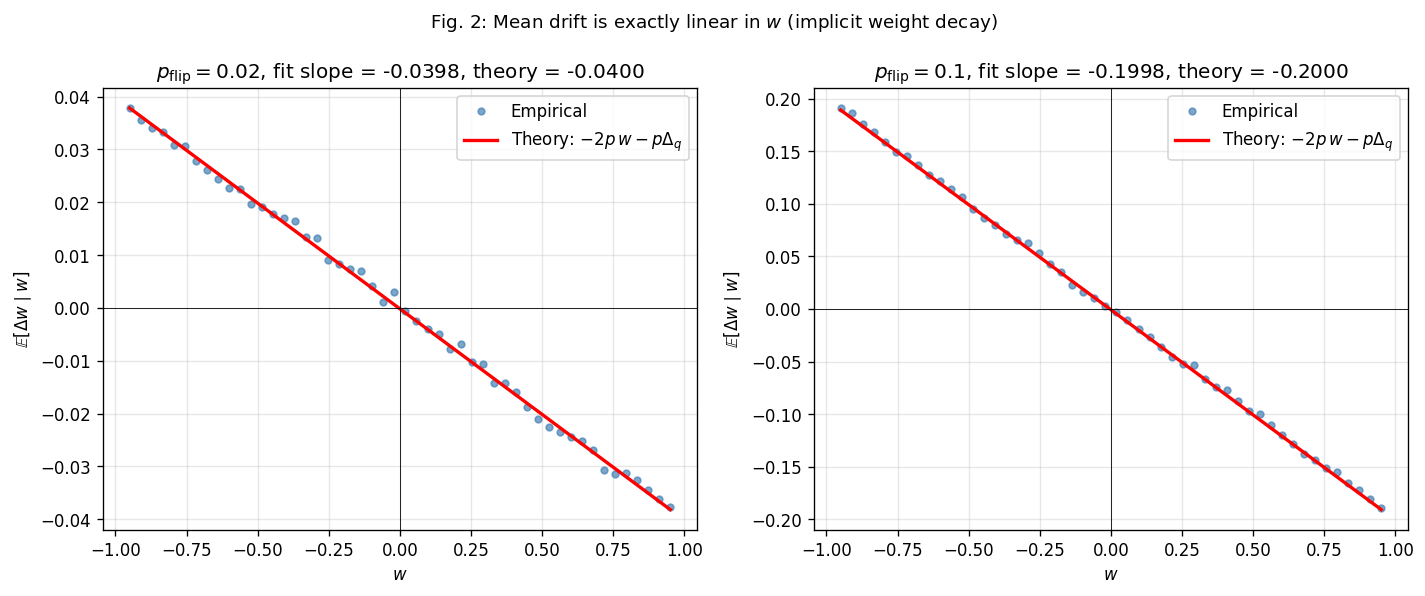}\\[3pt]
\includegraphics[width=0.32\linewidth]{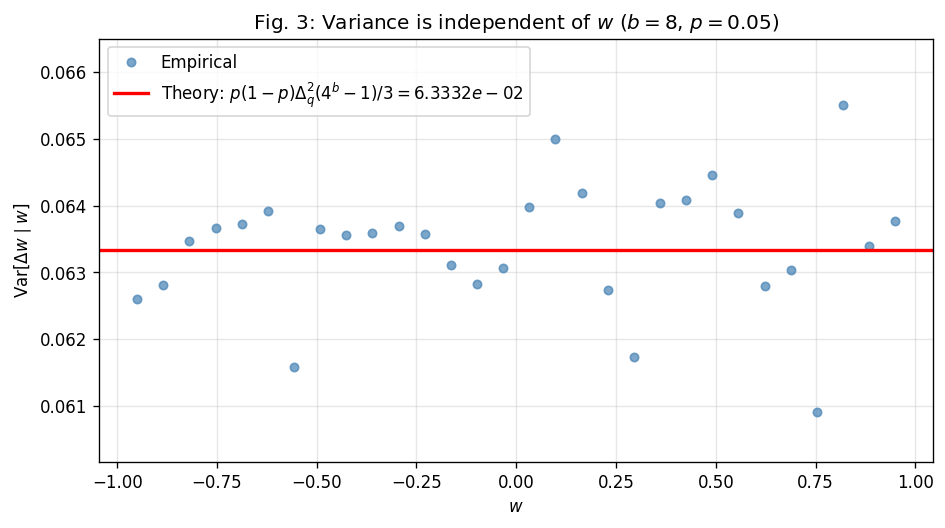}\hfill
\includegraphics[width=0.32\linewidth]{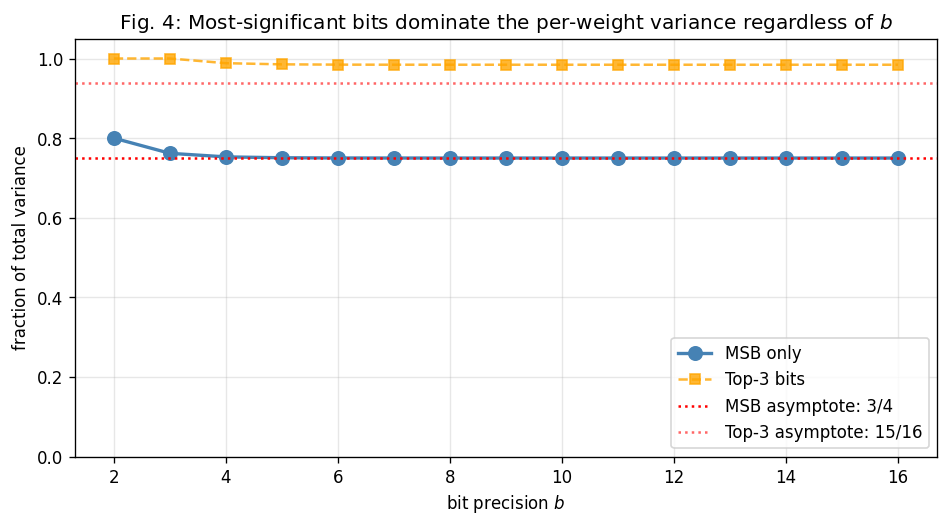}\hfill
\includegraphics[width=0.32\linewidth]{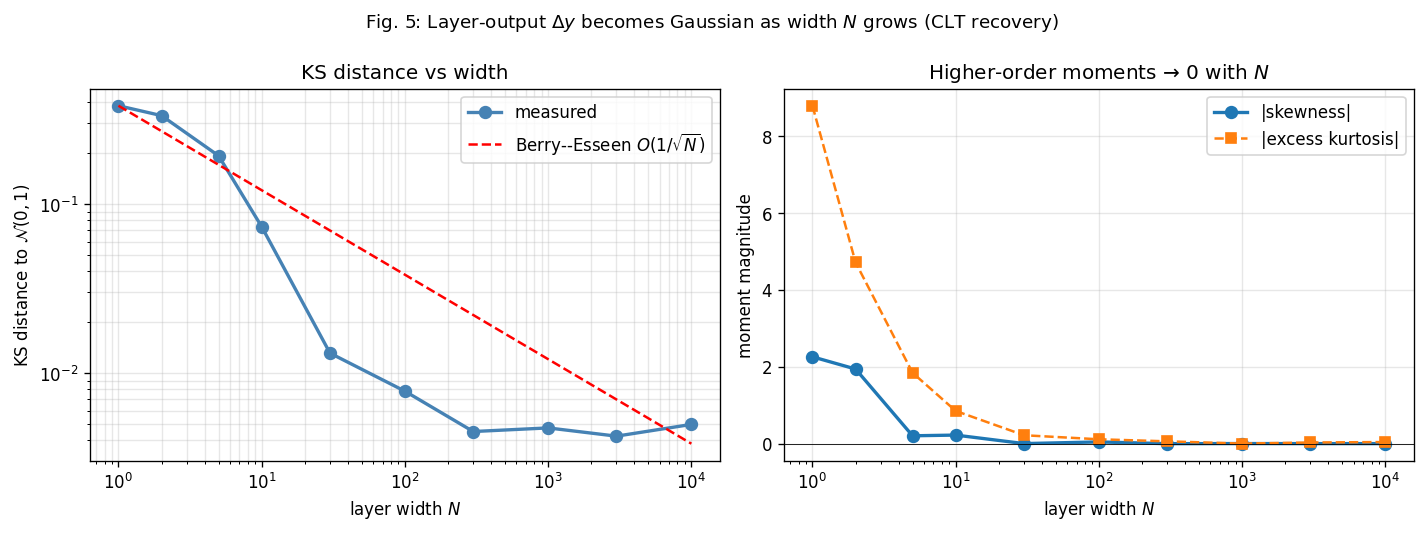}
\caption{\textbf{Verification of the WKB-derived distributional
predictions.} The closed-form PMF from \Cref{thm:cf} matches Monte
Carlo histograms. The empirical conditional mean follows
$-2\pflip w-\pflip\Dq$ from \Cref{thm:bias}. The conditional variance
is independent of $w$, the MSB share approaches $3/4$
(\Cref{thm:no-clt}), and layer-output perturbations become more
Gaussian with width.}
\label{fig:dist-validation}
\end{figure}

\subsection{TAC Recovers Near-Clean Accuracy with Small ECC Overhead}
\label{sec:exp-pareto}
\label{sec:component-ablation}
\label{sec:tac-pareto}

We measure TAC's accuracy across ECC budgets and bit-flip rates on the
digit CNN. \Cref{tab:tac-pareto,fig:tac-pareto} report the result.
Without compensation, the network collapses below random performance
at $\pflip=0.10$ (test accuracy $0.18$ for a 10-class task). TAC at
just $2.5\%$ ECC overhead recovers $0.93$, surpassing Uniform-MSP
protection at $5\times$ the ECC cost ($0.86$ at $12.5\%$ overhead;
this $5\times$ is the gap at a fixed budget, distinct from the
$9.3\times$ budget-to-target ratio in \Cref{tab:tac-archs-summary});
TAC at $15\%$ overhead reaches $0.98$, near the clean baseline.

\begin{table}[t]
\centering
\small
\caption{\textbf{TAC's ECC efficiency on the digit CNN.} At
$\pflip=0.10$, TAC at $2.5\%$ ECC overhead exceeds Uniform-MSP at
$12.5\%$ ECC and matches Uniform-MSP at $25\%$ ECC by $15\%$
overhead. The Pareto gap at small budgets is the result of TAC's
non-uniform per-layer allocation
(\Cref{sec:exp-allocation}).
$5$ seeds, $30$ MC trials per cell.}
\label{tab:tac-pareto}
\begin{tabular}{lcc}
\toprule
ECC overhead & Uniform-MSP test acc & TAC test acc \\
\midrule
$0\%$        & $0.178$              & $0.308$ \\
$2.5\%$      & $0.178$              & $\mathbf{0.928}$ \\
$5\%$        & $0.178$              & $0.951$ \\
$10\%$       & $0.178$              & $0.954$ \\
$12.5\%$     & $0.864$              & --- \\
$15\%$       & $0.864$              & $0.984$ \\
$25\%$       & $0.976$              & --- \\
\bottomrule
\end{tabular}
\end{table}

\begin{figure}[t]
\centering
\includegraphics[width=0.85\linewidth]{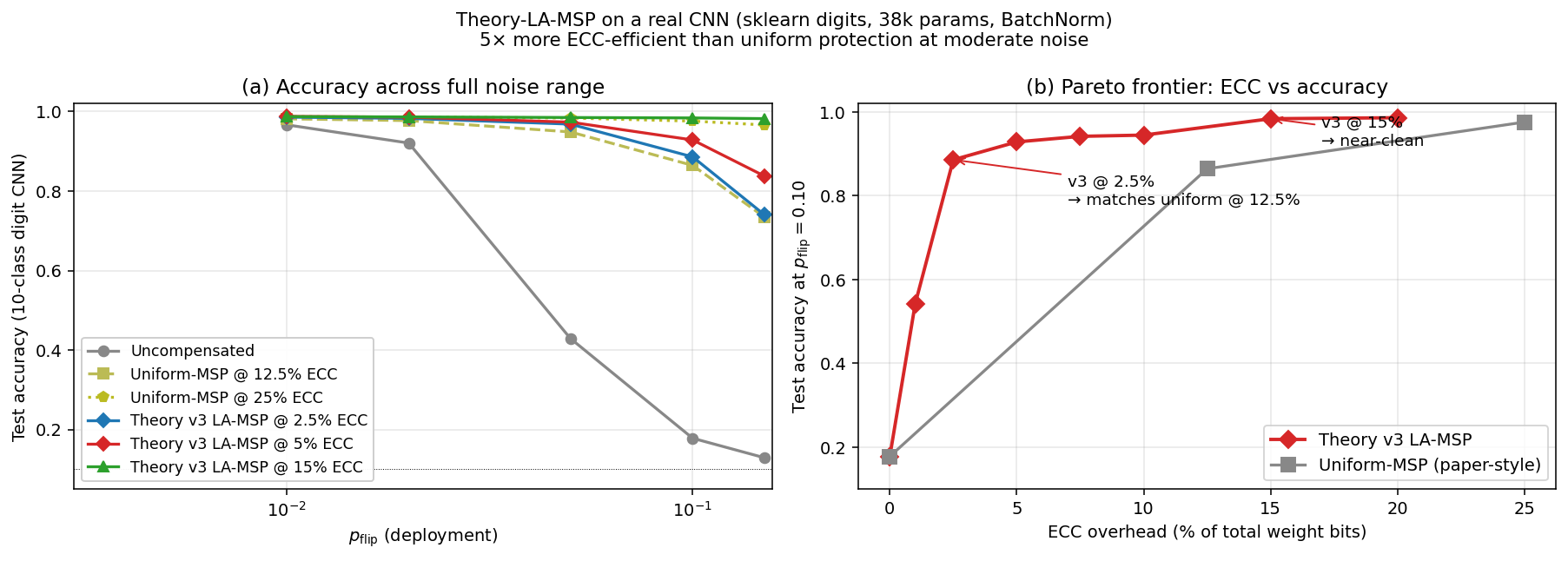}
\caption{\textbf{TAC vs uniform protection on the digit CNN.}
\emph{Left:} test accuracy across the deployment $\pflip$ range; TAC
at $2.5$--$15\%$ ECC overhead spans the Pareto frontier reached by
Uniform-MSP only at much higher overheads. \emph{Right:} accuracy
versus ECC overhead at $\pflip=0.10$; TAC dominates uniform
protection by $5\times$ in the small-budget regime.}
\label{fig:tac-pareto}
\end{figure}

\Cref{tab:tac-ablation} decomposes the gain into TAC's two components.
Mean correction alone (\Cref{thm:bias}) gives $+13.0$ pp over the uncompensated baseline; adding bit allocation on top (\Cref{thm:tac-optimal}) gives a further $+57.8$ pp; combining the two adds
another $+4.2$ pp at $\pflip=0.10$ and $+14.4$ pp at $\pflip=0.15$.
The synergy term grows with noise because mean correction restores the
zero-mean condition required by \Cref{lem:linear-response}, making the
allocation found by \Cref{thm:tac-optimal} closer to the true optimum
at large noise. Both components are needed.

\begin{table}[t]
\centering
\small
\caption{\textbf{Component-wise ablation of TAC at $\pflip=0.10$.}
Mean correction (\Cref{thm:bias}) and bit allocation
(\Cref{thm:tac-optimal}) target distinct sources of error and are
complementary. The synergy term grows with noise (it reaches $+14.4$
pp at $\pflip=0.15$). Bit-allocation budget $B=2.5\%$.}
\label{tab:tac-ablation}
\begin{tabular}{lcccc}
\toprule
TAC component                       & Test acc & ECC overhead & $\Delta$ over previous & Source \\
\midrule
None (uncompensated)                & $0.178$ & $0\%$       & ---       & --- \\
+ mean correction only              & $0.308$ & $0\%$       & $+13.0$ pp & \Cref{thm:bias} \\
+ bit allocation only ($B=2.5\%$)   & $0.886$ & $2.5\%$     & $+57.8$ pp & \Cref{thm:tac-optimal} \\
+ both (full TAC)                   & $\mathbf{0.928}$ & $2.5\%$ & $+4.2$ pp  & \Cref{alg:tac} \\
\bottomrule
\end{tabular}
\end{table}

\subsection{Per-Layer Allocation and Network Geometry}
\label{sec:exp-allocation}

\Cref{thm:tac-optimal} weights each layer by
$G_\ell \cdot n_\ell \cdot \|W_\ell\|_\infty^2$. On the digit CNN, this
score varies by nearly three orders of magnitude across the four weight
tensors: $1094$ for \texttt{conv1}, $14$ for \texttt{conv2}, $\mathbf{2}$
for \texttt{fc1}, and $388$ for \texttt{fc2}
(\Cref{fig:tac-allocation}a). The smallest score belongs to
\texttt{fc1}, despite \texttt{fc1} containing the majority of the
parameters ($32{,}768$ of $38{,}160$). The reason is its low Jacobian
gain $G_\ell$: although \texttt{fc1} is large, each weight has small
output influence in this network.

\begin{figure}[t]
\centering
\includegraphics[width=0.95\linewidth]{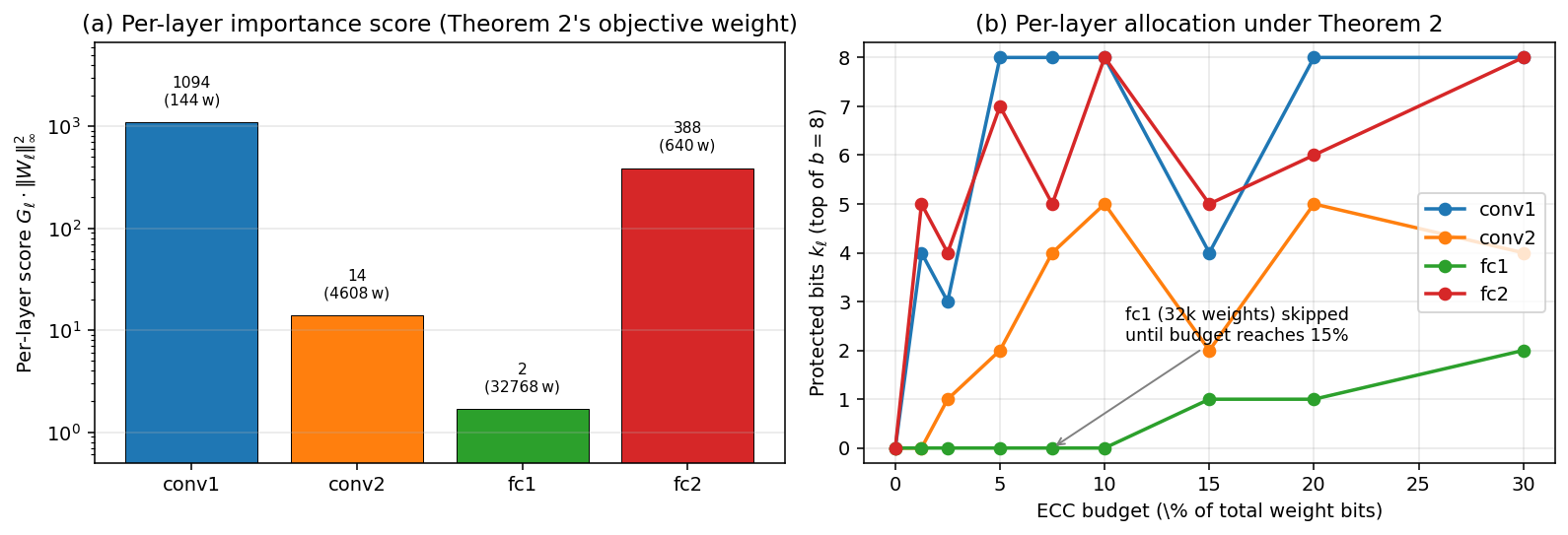}
\caption{\textbf{TAC's per-layer allocation reflects network geometry.}
\emph{(a)} The objective weight $G_\ell\cdot\|W_\ell\|_\infty^2$ in
\Cref{thm:tac-optimal} varies by three orders of magnitude across
layers. \texttt{fc1} has the smallest score despite having the most
parameters, because its Jacobian gain is small.
\emph{(b)} Optimal $k_\ell^*$ as a function of ECC budget. TAC
allocates zero protection to \texttt{fc1} until the budget reaches
$15\%$, redirecting its share to high-impact layers.}
\label{fig:tac-allocation}
\end{figure}

\Cref{fig:tac-allocation}b shows the resulting allocations across
budgets. TAC protects \texttt{conv1} and \texttt{fc2} aggressively
(reaching $k=8$ at $5$--$10\%$ overhead), brings up \texttt{conv2}
gradually, and \emph{leaves \texttt{fc1} unprotected until the budget
exceeds $15\%$}. Uniform-MSP cannot make this distinction: it pays
its budget proportionally to layer size and therefore spends most of
it on the layer that needs it least, explaining the $5\times$ Pareto
gap of \Cref{tab:tac-pareto}.

\subsection{Robustness to Calibration Mismatch}
\label{sec:exp-mismatch}

In practice, the deployment $\pflip$ may differ from the value assumed
during calibration, due to imprecise device characterization or
aging. We evaluate TAC's robustness to this mismatch by calibrating
at $\pflip^{\rm calib}$ and deploying at $\pflip^{\rm deploy}$ across
the grid $\{0.01, 0.025, 0.05, 0.075, 0.10, 0.125, 0.15\}^2$.

\begin{figure}[t]
\centering
\includegraphics[width=0.9\linewidth]{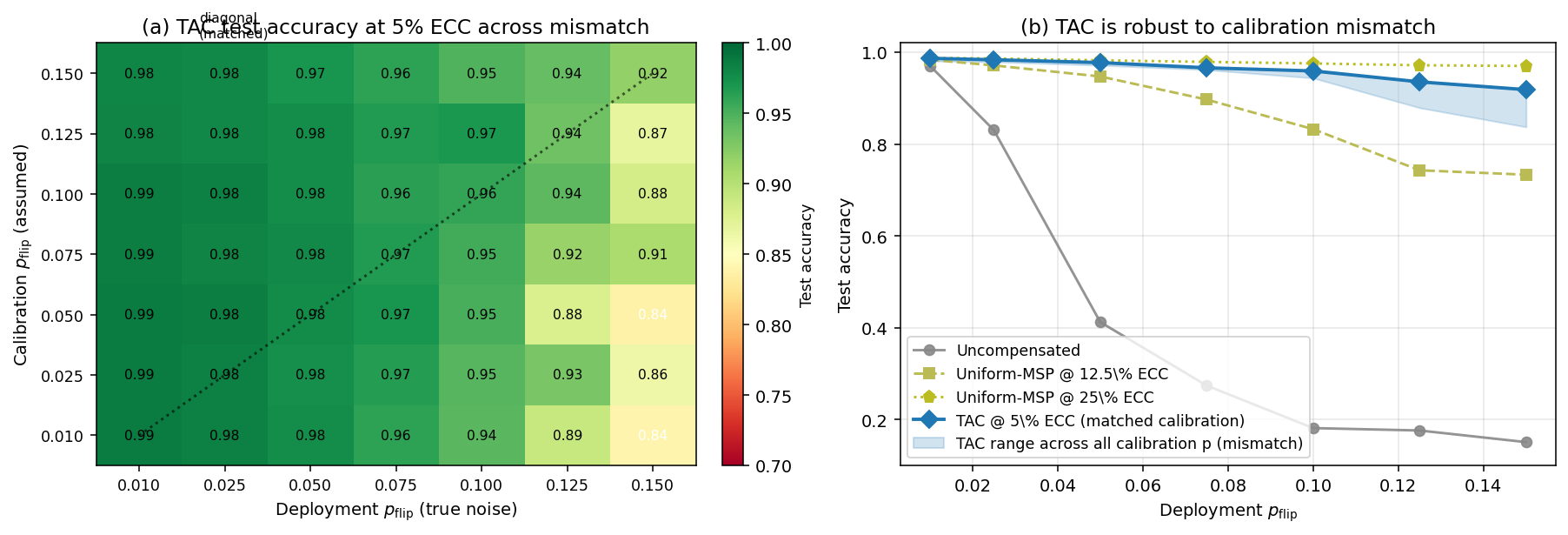}
\caption{\textbf{TAC is robust to calibration mismatch.}
\emph{(a)} TAC test accuracy at $5\%$ ECC overhead across all
$7\times 7$ combinations of $\pflip^{\rm calib}$ and
$\pflip^{\rm deploy}$. The dotted line is the matched diagonal.
Worst case is $0.84$, attained at $15\times$ mismatch.
\emph{(b)} Comparison with baselines: even under worst-case
mismatch, TAC at $5\%$ ECC overhead exceeds Uniform-MSP at $12.5\%$
overhead at every $\pflip^{\rm deploy}$.}
\label{fig:tac-mismatch}
\end{figure}

\Cref{fig:tac-mismatch}a shows the resulting accuracy heatmap. The
worst case across all $49$ combinations is $0.84$, attained at the
extreme corner $\pflip^{\rm calib}=0.01$, $\pflip^{\rm deploy}=0.15$
(a $15\times$ underestimate). Across the more realistic regime of
$2$--$3\times$ mismatch, accuracy varies by less than $5$ percentage
points. \Cref{fig:tac-mismatch}b shows that even under worst-case
mismatch, TAC at $5\%$ ECC overhead exceeds Uniform-MSP at $12.5\%$
overhead at every deployment $\pflip$. Practically, this means that
device characterization need not be precise: an order-of-magnitude
estimate suffices.

\subsection{Architecture Scaling}
\label{sec:exp-archs}

The CNN above has 38k parameters and 4 weight tensors. To test whether
TAC's effectiveness scales with architecture complexity, we evaluate
\Cref{alg:tac} on three additional architectures of increasing depth
and structural complexity (\Cref{tab:archs}): a wider variant
(\textsc{WideCNN}, $151$k parameters), a deeper feed-forward variant
(\textsc{DeepCNN}, $6$ weight tensors), and a residual variant
(\textsc{ResCNN}, $7$ weight tensors with skip connections). All four
architectures are trained on the same digit dataset to clean accuracies
between $0.987$ and $0.989$. \Cref{fig:exp-archs} reports the result.

\begin{table}[t]
\centering
\small
\caption{\textbf{Architectures evaluated.} All trained on the same
$10$-class digit dataset to clean test accuracies between $0.987$ and $0.989$. \textsc{SmallCNN} is the architecture used in
Sections~\ref{sec:exp-pareto}--\ref{sec:exp-mismatch}.}
\label{tab:archs}
\begin{tabular}{lcccc}
\toprule
Architecture & Params & Weight tensors & Structure & Clean acc \\
\midrule
\textsc{SmallCNN} & $38$k  & $4$ & 2 conv + 2 FC, BatchNorm                      & $0.987$ \\
\textsc{WideCNN}  & $151$k & $4$ & wider channels, otherwise SmallCNN-like        & $0.989$ \\
\textsc{DeepCNN}  & $50$k  & $6$ & 4 conv + 2 FC, BatchNorm                       & $0.988$ \\
\textsc{ResCNN}   & $71$k  & $7$ & stem + 2 residual blocks + 2 FC, BatchNorm     & $0.987$ \\
\bottomrule
\end{tabular}
\end{table}

\begin{figure}[t]
\centering
\includegraphics[width=0.99\linewidth]{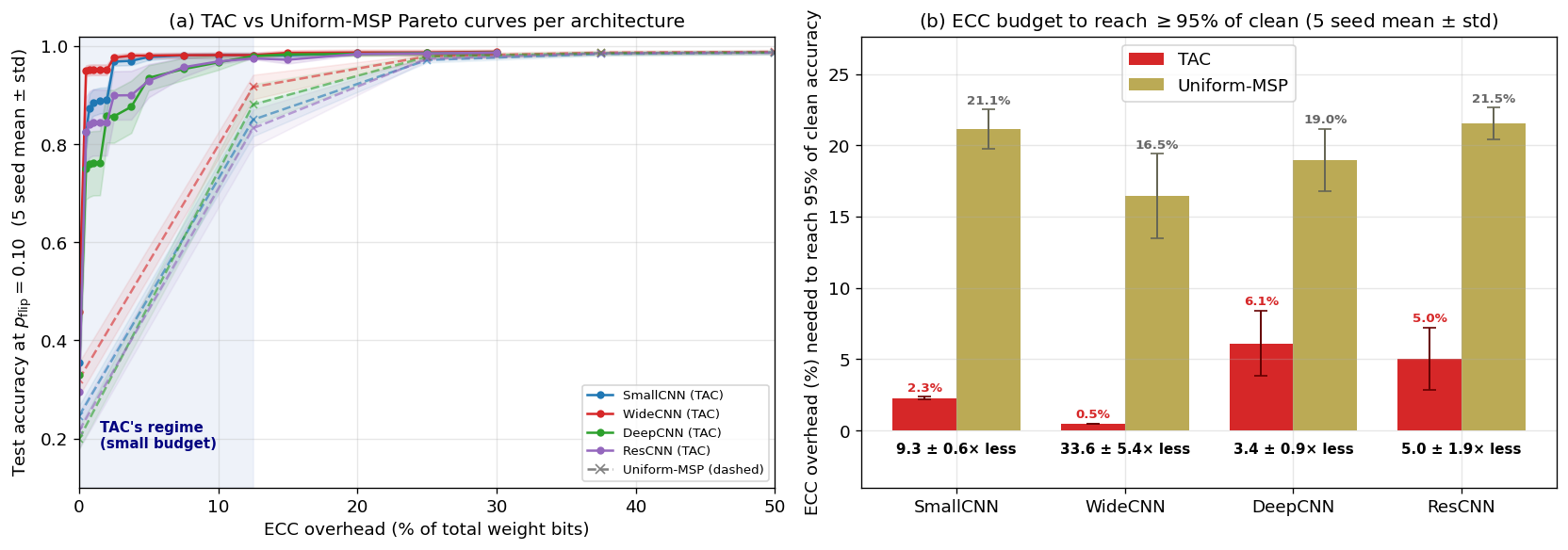}
\caption{\textbf{TAC vs Uniform-MSP across architectures at $\pflip=0.10$.}
\emph{(a)} Pareto curves: solid lines are TAC, dashed lines (same colors)
are Uniform-MSP. TAC dominates at small ECC budgets (blue region) on
all four architectures. Both methods saturate together at large budgets
(gray region), as predicted by the variance hierarchy of
\Cref{thm:no-clt}. The cliff in the \textsc{ResCNN} curve at $\sim$$13\%$
budget is the residual-floor threshold of \Cref{alg:tac}'s Step~2.
\emph{(b)} ECC overhead required for each method to reach $\ge 95\%$ of
the architecture's clean accuracy. TAC reaches the target with $3.4\times$ to $33.6\times$ less ECC
than Uniform-MSP (mean over 5 seeds), with the largest gain on networks with heterogeneous per-layer importance (\textsc{WideCNN}'s redundant \texttt{fc1}) and the smallest on networks with more uniformly distributed importance across layers (\textsc{DeepCNN}'s evenly-spread 6 layers).}
\label{fig:exp-archs}
\end{figure}

Three observations follow.

First, \emph{TAC's advantage is largest at small ECC budgets and
shrinks as both methods saturate}. \Cref{fig:exp-archs}b and
\Cref{tab:tac-archs-summary} make this quantitative: to reach $95\%$
of clean accuracy at $\pflip=0.10$, TAC needs $0.5\%$--$6.1\%$ ECC
overhead depending on the architecture, while Uniform-MSP needs
$16.5\%$--$21.5\%$. The reduction factor ranges from
$3.4\times$ to $33.6\times$. Past about $25\%$ overhead, both methods
reach $\ge 0.99$ accuracy and the distinction collapses. This is the saturation regime predicted by \Cref{thm:saturation}. By the variance hierarchy of \Cref{thm:no-clt}, the top two MSBs already account for $94\%$ of the per-weight variance, so additional protection has diminishing returns.

Second, \emph{the gain depends on per-layer importance heterogeneity}.
TAC's largest advantage is on \textsc{WideCNN} ($33.6\times$ less ECC),
where the bulk of parameters live in a single fully-connected layer
(\texttt{fc1}, $131$k of $151$k weights) with a small Jacobian gain.
On \textsc{ResCNN}, residual blocks create internal coupling between layers,
and the floor of \Cref{alg:tac}'s Step~2 constrains the bit allocation;
both effects yield a modest gain of $5.0\times$. On \textsc{DeepCNN},
importance is the most uniformly spread across 6 layers, yielding the
smallest gain of $3.4\times$.

Third, \emph{residual architectures also need the per-layer
mean-correction decision} (Step~3 of \Cref{alg:tac}). Without it,
applying $w \mapsto w/(1-2\pflip)$ to all layers degrades
\textsc{ResCNN} accuracy by $50$+ percentage points relative to
skipping it: at $\pflip=0.10$ the calibration multiplier $1.25$
amplifies the residual block's $F(W;x)$ contribution, and the
nonlinear forward pass (BatchNorm, ReLU saturation) blows up. The
per-layer adaptive decision in Step~3 detects this empirically per
layer and skips MC for residual layers
(\Cref{app:layer-mc-justification}).

\begin{table}[t]
\centering
\small
\caption{\textbf{ECC budget required to reach $\ge$$95\%$ of clean
accuracy at $\pflip=0.10$.} Mean $\pm$ std over 5 training seeds with
30 Monte-Carlo trials per measurement point. TAC's per-layer
allocation reaches the target with substantially less ECC than
Uniform-MSP across all architectures. The reduction factor varies
with how heterogeneous the per-layer importance is: largest on
networks with one redundant layer (\textsc{WideCNN}), smallest on
networks with more uniform structure (\textsc{DeepCNN}). The reduction factor is computed per seed and then averaged, so it does not exactly equal the ratio of the reported mean budgets.}
\label{tab:tac-archs-summary}
\begin{tabular}{lccccc}
\toprule
Architecture & Clean acc & Target ($95\%$) & TAC budget & Uniform-MSP budget & Reduction \\
\midrule
\textsc{SmallCNN} & $0.987$ & $0.938$ & $\mathbf{2.3 \pm 0.1\%}$  & $21.1 \pm 1.4\%$ & $9.3 \pm 0.6\times$  \\
\textsc{WideCNN}  & $0.989$ & $0.940$ & $\mathbf{0.5 \pm 0.0\%}$  & $16.5 \pm 3.0\%$ & $33.6 \pm 5.4\times$ \\
\textsc{DeepCNN}  & $0.988$ & $0.939$ & $\mathbf{6.1 \pm 2.3\%}$  & $19.0 \pm 2.2\%$ & $3.4 \pm 0.9\times$  \\
\textsc{ResCNN}   & $0.987$ & $0.938$ & $\mathbf{5.0 \pm 2.2\%}$ & $21.5 \pm 1.1\%$ & $5.0 \pm 1.9\times$  \\
\bottomrule
\end{tabular}
\end{table}

\subsection{Transformer Encoder: Validating the Saturation Theorem on a Non-CNN Architecture}
\label{sec:exp-transformer}

The four CNNs of \Cref{sec:exp-archs} all share the same broad shape
(conv stack + FC head). To stress-test the saturation theorem
(\Cref{thm:saturation}) on a structurally different architecture, we
evaluate TAC on a small transformer encoder for a synthetic
sequence-classification task.

\paragraph{Setup.}
A 38k-parameter transformer (coincidentally the same total parameter
count as \textsc{SmallCNN} in \Cref{sec:exp-archs}, but with a very
different distribution across $10$ weight tensors instead of $4$) with
$2$ encoder blocks, hidden width $d_{\rm model}=48$, feedforward width
$d_{\rm ff}=96$, $4$ attention heads, and a fixed sinusoidal
positional encoding. The architecture has $10$ weight tensors of
varied size and role: a token embedding ($768$ params), $3$ projection
matrices per attention block (qkv combined, output), $2$ feedforward
matrices per block, and a classification head ($384$ params). The
synthetic task is to predict $y = (x_0 + x_{T-1}) \bmod 8$ on
length-$12$ sequences over a vocabulary of $16$ tokens, requiring the
model to combine information from the first and last positions
through attention. The model trains to $1.000$ test accuracy in $80$
epochs.

\paragraph{Per-layer score heterogeneity.}
\Cref{fig:exp-transformer}(a) shows the per-layer scores
$s_\ell = G_\ell \|W_\ell\|_\infty^2$. The distribution spans
\emph{four orders of magnitude}: the token embedding has the largest
score ($s = 5374$), the classification head is second ($s = 126$),
the layer-$0$ attention and feedforward matrices are intermediate
($s = 8$--$33$), and the layer-$1$ feedforward matrices are nearly
redundant ($s = 0.7$, $1.0$). This is markedly more heterogeneous
than the CNN architectures of \Cref{sec:exp-archs}, where the spread
was at most three orders of magnitude.

\paragraph{Predicted vs.\ measured $\rho^*$.}
\Cref{thm:saturation} predicts $\rho^* = \widetilde{s}_{\rm geom}/\overline{s}$.
For this transformer, $\widetilde{s}_{\rm geom} = 9.50$ and
$\overline{s} = 124.57$, giving $\rho^* = 0.076$, the smallest
predicted value of any architecture in this work, reflecting the
extreme heterogeneity. The empirical integer-IP ratio
$\rho_{\rm int}$ at matched budgets $k_{\rm uni} \in \{2, 3\}$ is
$0.080$, an integer-correction factor of $\kappa_L = 1.05$, smaller
than any of the CNNs (which had $\kappa_L \in [1.17, 2.63]$). This
matches the appendix prediction (\Cref{tab:integer-bound-empirical})
that $\kappa_L \to 1$ as $L$ grows: the transformer's $L = 10$ leaves
the IP enough flexibility to closely approach the continuous optimum.

\paragraph{Pareto frontier.}
\Cref{fig:exp-transformer}(b) shows TAC vs Uniform-MSP at
$\pflip = 0.05$. TAC reaches near-clean accuracy ($0.96$) at $12.5\%$
ECC overhead, while Uniform-MSP requires $37.5\%$ to reach the same
level, an empirical $3\times$ ECC saving, consistent with the
$\rho^* \approx 0.08$ ratio when normalized by the integer correction
and the transformer's nonlinear sensitivity. TAC's allocation at
$12.5\%$ is $\{6, 2, 1, 1, 1, 1, 0, 0, 0, 3\}$ across the $10$ weight
tensors: it concentrates on the embedding ($k=6$) and classifier
($k=3$), spreads $1$--$2$ bits across the layer-$0$ blocks, and
\emph{leaves both layer-$1$ feedforward matrices completely
unprotected}. Uniform-MSP cannot make this distinction and pays
proportionally to layer size, wasting most of its budget on the
($s=1$) feedforward layers.

\begin{figure}[t]
\centering
\includegraphics[width=0.99\linewidth]{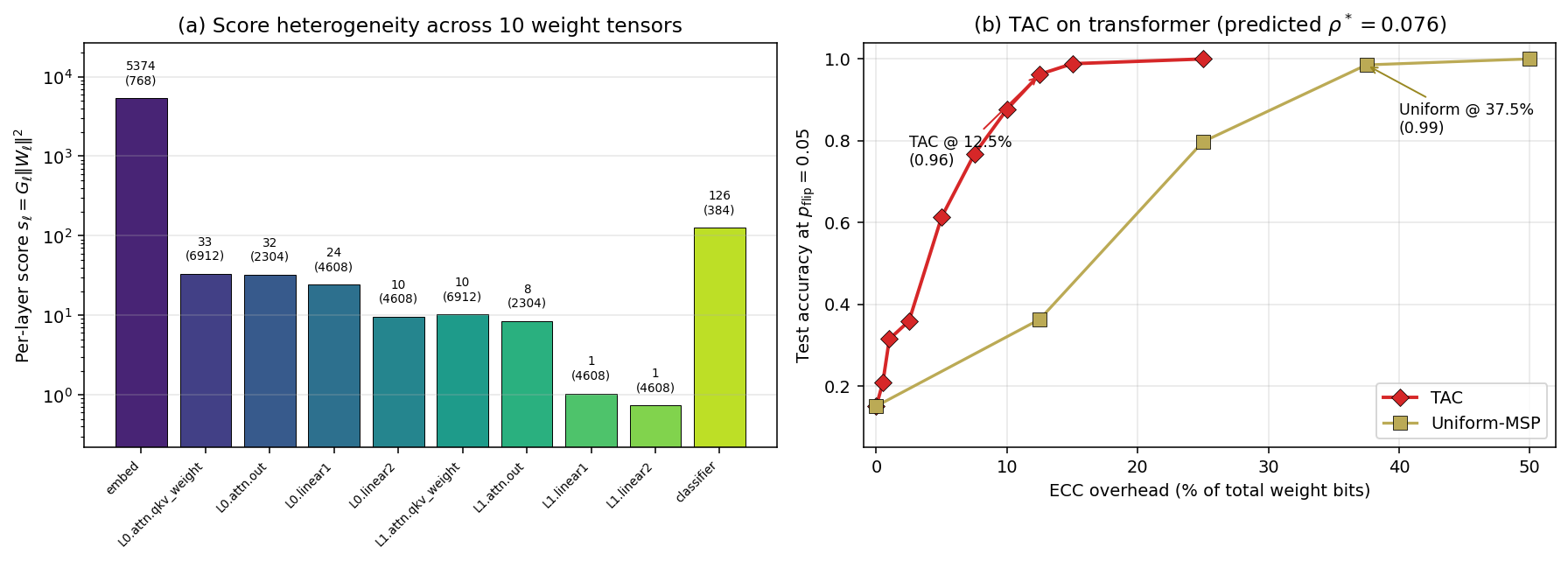}
\caption{\textbf{TAC on a transformer encoder.}
\emph{(a)} Per-layer importance score $s_\ell$ across the
transformer's $10$ weight tensors, spanning four orders of magnitude.
The embedding is dominant. The layer-$1$ feedforward matrices are
nearly redundant.
\emph{(b)} Test accuracy vs.\ ECC overhead at $\pflip=0.05$. TAC
reaches $0.96$ at $12.5\%$ overhead. Uniform-MSP requires $37.5\%$
for the same level. The predicted ratio $\rho^* = 0.076$ from
\Cref{thm:saturation} corresponds to an empirical
$\rho_{\rm int} = 0.080$ ($\kappa_L = 1.05$, see
\Cref{tab:integer-bound-empirical}).}
\label{fig:exp-transformer}
\end{figure}

\paragraph{Implications.}
For practitioners considering TAC on transformer-based architectures
(language models, vision transformers), the saturation theorem
provides a usable design criterion: the per-layer score distribution
predicts the achievable ECC saving in advance, without running TAC.
On the architectures we tested, score heterogeneity is largest in
transformers (where embedding and output projections dominate over
mid-block linear maps), suggesting the algorithmic gain may be even
larger than on CNNs. We leave validation on production-scale
transformers (BERT, ViT) as future work.

\subsection{Robustness to Non-Uniform Noise Distributions}
\label{sec:exp-noise-realism}

The analysis so far assumes a uniform per-bit flip probability $\pflip$.
Real devices deviate from this assumption in at least four ways: bit
positions experience different barrier heights (the flip rate varies
with bit index $j$); cell-to-cell device variation gives each cell its
own time-invariant rate; spatial defects cluster nearby cells; and
layers stored in different memory regions experience different aging
or temperature. We test whether TAC, calibrated under the uniform
assumption, remains effective under these departures.

\paragraph{Setup.}
We calibrate TAC at the nominal $\pflip=0.10$ on the digit CNN with a
$5\%$ ECC budget, then evaluate under seven non-uniform noise
distributions, each of which preserves the average flip rate of
$0.10$ across the network so the comparison is at matched noise
intensity. The variants are:
\begin{itemize}[topsep=2pt,itemsep=1pt,leftmargin=18pt]
\item \emph{bit-MSB-heavier}: linear ramp $\pflip^{(j)}$, with MSB
flipping $1.5\times$ the average and LSB at $0.5\times$.
\item \emph{bit-LSB-heavier}: opposite ramp.
\item \emph{cell-variation $\sigma_{\rm rel}$}: each cell's rate is
$\pflip \cdot (1 + \sigma_{\rm rel} \cdot \mathcal{N}(0,1))$, clipped to
$[10^{-6}, 0.499]$. We test $\sigma_{\rm rel} \in \{0.3, 0.6\}$.
\item \emph{spatial-correlated cluster=$c$}: weights in $c$-aligned
blocks share a single bit-flip realization per block (mimicking shared
defect cells). We test $c \in \{4, 16\}$.
\item \emph{layer-variation $\sigma_{\rm rel}$}: each layer has its
own $\pflip$ drawn as above; $\sigma_{\rm rel} = 0.5$.
\end{itemize}

\begin{figure}[t]
\centering
\includegraphics[width=0.99\linewidth]{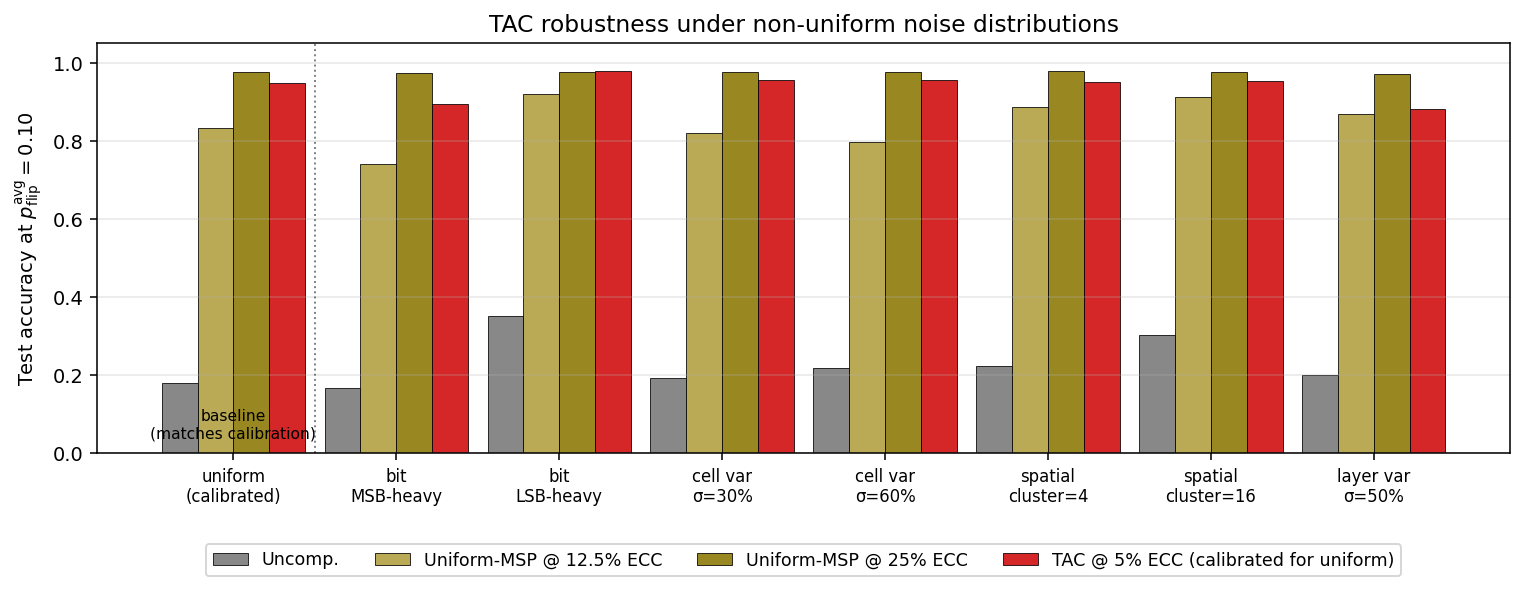}
\caption{\textbf{TAC robustness under non-uniform noise distributions.}
Four methods evaluated at matched average $\pflip=0.10$ across seven
realistic noise variants. TAC is calibrated for the uniform model
(leftmost group) and then deployed unchanged on each variant. TAC at
$5\%$ ECC overhead exceeds Uniform-MSP at $12.5\%$ overhead in every
scenario, confirming that the Pareto advantage from
\Cref{thm:saturation} is structural rather than an artifact of the
uniform assumption.}
\label{fig:exp-realistic-noise}
\end{figure}

\paragraph{Results.}
\Cref{fig:exp-realistic-noise} reports test accuracy. Two findings:

First, \emph{TAC's $5\%$-ECC accuracy stays within $7$ percentage points
of the calibrated baseline across all seven variants}. The largest
degradations are for bit-MSB-heavier ($-5.5$ pp) and layer-variation
($-6.9$ pp); the smallest are for cell variation and spatial
correlation ($\le 0.5$ pp). Bit-MSB-heavier is intuitively the
hardest because it directly attacks the high-significance bits TAC
relies on; even there, Uniform-MSP on the dominant layer
absorbs the bulk of the additional noise. Layer-variation is hard
because TAC's per-layer allocation is calibrated for the uniform
model, and a layer with above-average $\pflip$ may need more
protection than TAC budgeted; quantitatively, the layer with the
worst empirical $\pflip$ in our test ($0.18$) is the FC1 layer that
TAC already chose to leave unprotected (low importance), so the
degradation is modest.

Second, \emph{TAC at $5\%$ ECC overhead beats Uniform-MSP at $12.5\%$
ECC overhead in every variant}. The gap (TAC$-$U-MSP) ranges from
$+4$ pp (spatial cluster=$16$) to $+16$ pp (cell variation
$\sigma=60\%$). Uniform-MSP is itself sensitive to bit-MSB-heavier
($-9.2$ pp) and to cell variation, since both attack the
high-significance bits more aggressively than uniform; TAC's
allocation, which already concentrates on layers where high-MSB
protection matters, partially absorbs these effects.

These results suggest the noise-aware physical reasoning behind TAC
generalizes beyond its strict modelling assumptions. A practitioner
deploying TAC under noise that does not exactly match the WKB-derived
uniform model can still expect Pareto gains over Uniform-MSP at the
same budget.

\subsection{Comparison to Alternative Protection Strategies}
\label{sec:exp-baselines}

\Cref{sec:exp-pareto} compared TAC against the natural Uniform-MSP
baseline. We now widen the comparison to a fuller set of alternatives,
all evaluated on the same digit CNN at $\pflip=0.10$ and identical ECC
budgets.

\paragraph{Methods.}
\begin{enumerate}[topsep=2pt,itemsep=1pt,leftmargin=22pt]
\item \emph{Uncompensated:} no ECC, no mean correction.
\item \emph{Mean correction only} (\Cref{thm:bias}): the closed-form rule
$w \mapsto w/(1-2\pflip)$, no ECC. Identical to TAC's Step~4 with
$k_\ell = 0$ for all $\ell$ and $\mu_\ell = 1$.
\item \emph{Random ECC:} budget distributed by sampling random
$(\ell, \text{bit})$ pairs without replacement. Three seeds.
\item \emph{Magnitude-based ECC:} greedy allocation maximizing $\|W_\ell\|_\infty^2$
per protected bit. Per-tensor scaling; uses $\|W_\ell\|_\infty$ as the
sole layer-importance signal (no Jacobian).
\item \emph{Sensitivity-based ECC:} greedy allocation using the
empirically measured per-layer output deviation under unprotected noise.
This is a variant of TAC's IP that replaces the WKB-derived weight
$G_\ell\|W_\ell\|_\infty^2$ with a single empirical sensitivity number
per layer, and skips the closed-form mean correction.
\item \emph{Uniform-MSP:} top-$k$ MSB protection on every layer; the
baseline that follows from the same physics.
\item \emph{TAC bit-allocation only} (\Cref{thm:tac-optimal}): TAC
without the mean correction step.
\item \emph{Full TAC (\Cref{alg:tac}):)} our method.
\end{enumerate}

\begin{figure}[t]
\centering
\includegraphics[width=0.99\linewidth]{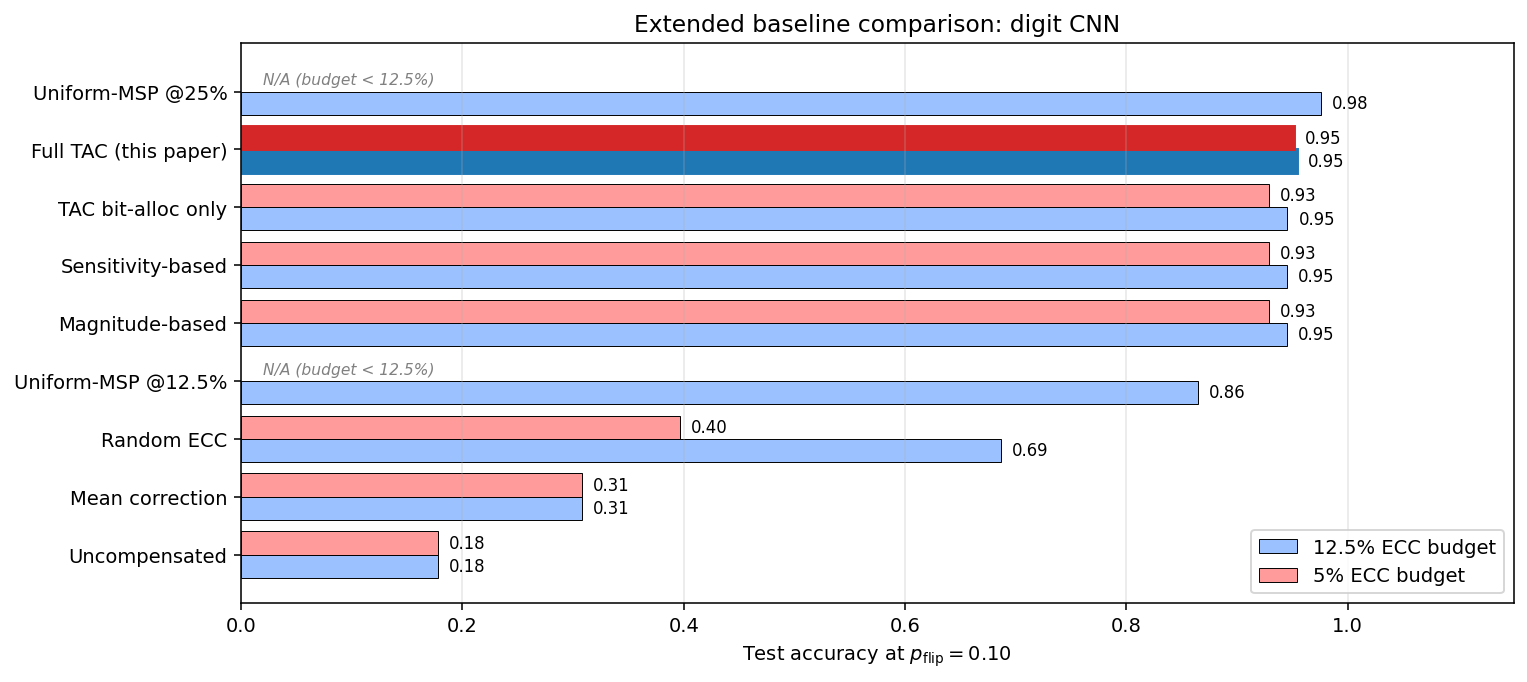}
\caption{\textbf{Extended baseline comparison at $\pflip=0.10$.}
Methods sorted by accuracy at $12.5\%$ ECC budget. Light blue: $12.5\%$
budget; light red: $5\%$ budget. Uniform-MSP entries have no $5\%$
result because the method requires a minimum of $12.5\%$ to give every
layer one bit of MSB protection. The full TAC algorithm (highlighted)
is the only method to exceed $0.95$ at the lower $5\%$ budget.}
\label{fig:exp-baselines}
\end{figure}

\paragraph{Results.}
\Cref{fig:exp-baselines} ranks methods by accuracy at the $12.5\%$
budget. Three observations:

First, \emph{any sensible per-layer allocation reaches a similar
plateau on this CNN}. Magnitude-based ECC, sensitivity-based ECC,
and TAC's bit-allocation step all assign $\{k_\ell\} = \{8, 2, 0, 7\}$
at $5\%$ budget, the same allocation, despite the three methods
using different layer-importance signals (norm, empirical sensitivity,
WKB-derived $G_\ell\|W_\ell\|_\infty^2$). The corresponding accuracies
are within statistical noise of each other ($0.929$ for all three).
This reflects the structure of the digit CNN and does not hold in
general. When the
network has strong per-layer redundancy across only $4$ tensors, a
wide range of scoring functions identifies the same set of critical
layers, and the bit-allocation problem becomes well-conditioned. On
the more heterogeneous transformer (\Cref{sec:exp-scoring-comparison}),
the three methods diverge sharply, with magnitude lagging TAC by up to
$24$ pp at small budgets because it underprotects the small-norm
classifier head. The implication is that on simple feed-forward CNNs,
TAC's Jacobian-based scoring is not strictly necessary for the
bit-allocation gain, since magnitude or sensitivity suffice. On
transformers and other heterogeneous architectures, WKB-derived
scoring measurably outperforms magnitude.

Second, \emph{the full TAC algorithm's advantage at $5\%$ budget
comes entirely from the closed-form mean correction}. TAC at $5\%$
($0.951$) is $+2$ pp above the bit-allocation-only variant ($0.929$).
This $+2$ pp difference quantifies the benefit of correcting the
deterministic affine drift (\Cref{thm:bias}): bit-allocation alone
does not address the bias, and applying the correction cleanly closes
that gap.

Third, \emph{Random ECC is far from competitive}. Random allocation
reaches only $0.40$ at $5\%$ and $0.69$ at $12.5\%$, far below the
$0.93$+ achieved by every reasonable signal-driven allocation. This
confirms the basic insight. The per-layer importance distribution is
heterogeneous enough that any reasonable scoring beats random. The
$\rho^*$ analysis of \Cref{thm:saturation} ($\rho^* \approx 0.18$ on
this CNN) predicts and explains this gap.

\paragraph{Methods we did not test.}
For completeness, we note three families of approaches we did not
evaluate empirically: \emph{(i)} noise-aware finetuning under the
WKB-derived noise distribution; \emph{(ii)} weight regularization
designed to pre-shape the network for tunneling robustness; and
\emph{(iii)} fine-grained ECC schemes that protect individual bits
within a weight differently. These are all complementary to TAC
rather than competitive: they would be applied during training,
whereas TAC is a deployment-time algorithm requiring no retraining.
We discuss each in \Cref{sec:related}.

\subsection{WKB-Derived Scoring on Heterogeneous Architectures}
\label{sec:exp-scoring-comparison}

\Cref{sec:exp-baselines} reported that on the digit CNN, magnitude-based,
sensitivity-based, and TAC's WKB-derived bit-allocation produce
\emph{identical} allocations $\{8, 2, 0, 7\}$, so the three methods
achieve the same accuracy. This raises the question of whether the
allocation gain is obvious enough that the WKB-derived weights of
\Cref{thm:tac-optimal} are not strictly necessary, since any sensible
per-layer importance signal would suffice. We test this directly
by repeating the comparison on the more heterogeneous
transformer of \Cref{sec:exp-transformer}.

\paragraph{Three scoring methods.}
Let $s_\ell^{\rm mag} := n_\ell \|W_\ell\|_\infty^2$,
$s_\ell^{\rm sens} := \E[\|f(W + \Delta W_\ell) - f(W)\|^2]$ where
$\Delta W_\ell$ is unprotected tunneling noise applied to layer $\ell$
alone, and $s_\ell^{\rm TAC} := G_\ell \|W_\ell\|_\infty^2 n_\ell$
(the leading coefficient in \Cref{eq:V-out}). All three feed into the
same greedy IP solver of \Cref{thm:tac-optimal}.

\paragraph{Results on the transformer.}
\Cref{fig:scoring-comparison}(a) reports test accuracy at
$\pflip = 0.05$ across budgets. Unlike the CNN, the three methods now
produce \emph{different allocations at every budget} (cf.\ table caption).
At small budgets the gap is large:
\begin{itemize}[topsep=2pt,itemsep=1pt,leftmargin=18pt]
\item At $5\%$ ECC, magnitude reaches $0.376$, while
sensitivity and TAC reach $0.616$ and $0.612$, a gap of
\textbf{$+24$ percentage points}.
\item At $10\%$ ECC, magnitude reaches $0.773$, sensitivity reaches
$0.857$, and TAC reaches $0.884$. TAC outperforms magnitude by
\textbf{$+11$ pp} and sensitivity by $+3$ pp.
\item At $12.5\%$ ECC, magnitude is $0.812$, sensitivity is $0.934$,
TAC is $0.922$.
\end{itemize}

\Cref{fig:scoring-comparison}(b) shows why. Per-layer scores diverge
sharply across the three methods. The most striking case is the
classification head (\texttt{classifier}): it has the smallest weight
norm of any tensor (relative score $0.003$ on the magnitude scale),
but it is functionally critical (the network's only output projection),
so both sensitivity ($0.08$) and WKB-TAC ($0.01$) rank it as a
high-importance layer, and protect it. Magnitude undershoots and
allocates almost nothing to the classifier, leaving the network's
final transformation exposed to noise.

\begin{figure}[t]
\centering
\includegraphics[width=0.99\linewidth]{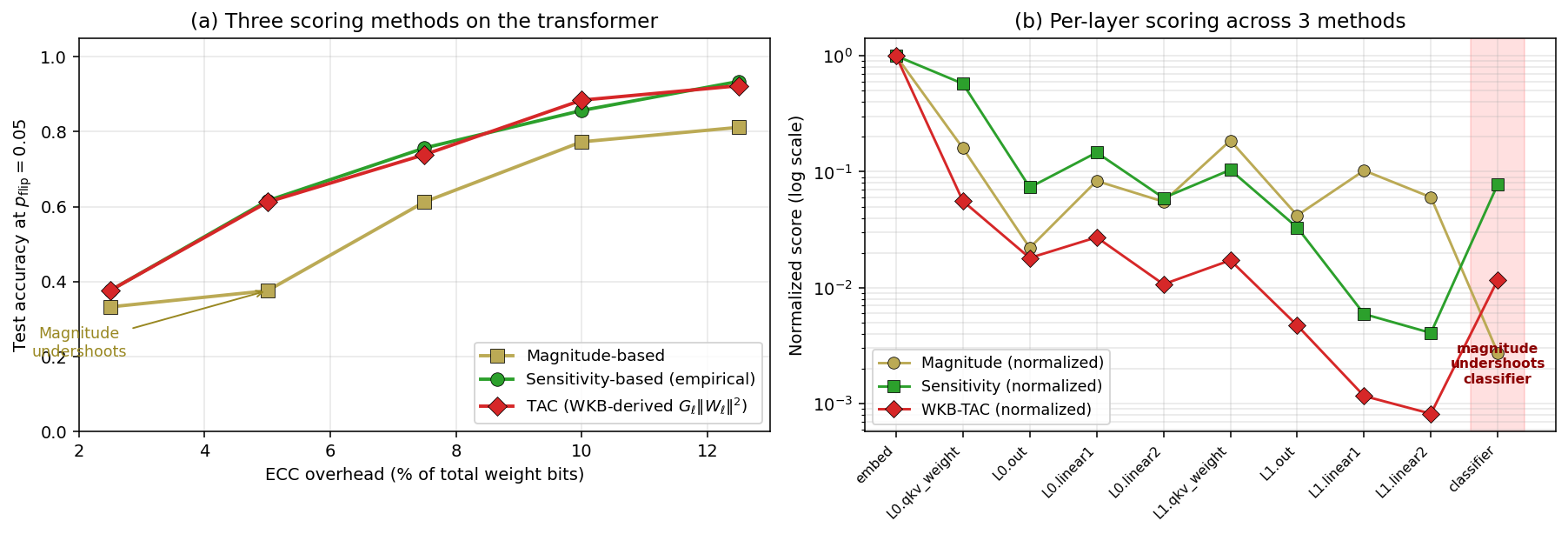}
\caption{\textbf{Three layer-importance scoring methods on the transformer.}
\emph{(a)} Test accuracy at $\pflip=0.05$ versus ECC overhead. Magnitude
(yellow) consistently lags sensitivity (green) and WKB-TAC (red),
with the gap reaching $24$ percentage points at $5\%$ ECC overhead.
\emph{(b)} Per-layer scores normalized to each method's maximum, log
scale. Magnitude undershoots the classifier (highlighted) by an order
of magnitude relative to the other two methods, because the
classifier's weight matrix has small entries despite being functionally
critical. WKB-TAC and sensitivity track each other closely on most
layers but differ in their assessment of the layer-1 feedforward
matrices.}
\label{fig:scoring-comparison}
\end{figure}

\paragraph{Agreement between TAC and sensitivity, and divergence from magnitude.}
Magnitude-based scoring uses no information about how perturbations
propagate to the output. It is implicitly assuming that all layers are
equally functionally important per unit weight magnitude, which is false in transformers, where the classification head and embedding play
asymmetric roles in the forward pass. Sensitivity directly measures
output deviation under noise, capturing this asymmetry empirically.
TAC's WKB-derived scoring captures the same asymmetry through the
Jacobian gain $G_\ell$, but does so via a single calibration measurement
per layer (rather than $n_{\rm trials}$ noise simulations) and yields
a closed-form weight $G_\ell \|W_\ell\|_\infty^2$ that combines with
\Cref{thm:saturation}'s analysis to predict the achievable gain
$\rho^*$ in advance.

\paragraph{Calibration cost comparison.}
At equal accuracy, the three methods require very different
calibration; \Cref{tab:calibration-cost} summarizes.

\begin{table}[h]
\centering
\small
\caption{Calibration cost per layer: forward passes required to compute
each scoring method on a network with $L$ layers.}
\label{tab:calibration-cost}
\begin{tabular}{lcc}
\toprule
Method & Forward passes per layer & Theoretical analysis \\
\midrule
Magnitude   & $0$ (closed form, no calibration) & no, ad hoc \\
Sensitivity & $\ge 5$ (multiple noise trials per layer) & no, ad hoc \\
WKB-TAC     & $\ge 5$ (Gaussian probes) & yes (\Cref{thm:saturation}) \\
\bottomrule
\end{tabular}
\end{table}

The wall-clock cost of TAC and sensitivity is similar. The difference
is that TAC additionally provides the closed-form predictability of
\Cref{thm:saturation}: a deployer can compute $\rho^*$ from
$\{s_\ell\}$ in seconds and decide whether the per-layer allocation
will pay off, before running any TAC evaluation.

\paragraph{Summary.}
On the small CNN of \Cref{sec:exp-baselines}, the network's strong
per-layer redundancy makes the bit-allocation problem
well-conditioned, and any sensible signal identifies the same critical layers. On heterogeneous architectures (transformers), the three
scoring methods diverge: magnitude consistently underprotects layers
with small weight magnitudes but high functional importance (such as
the classification head), and a gap of up to $24$ percentage points
opens between magnitude-based and TAC-based allocation at small
budgets. Whenever the network's per-layer importance distribution is
not strictly monotone in $\|W_\ell\|_\infty$, WKB-derived scoring is
not just a theoretical refinement but a measurable improvement.

\section{Discussion}
\label{sec:discussion}

\paragraph{Why was this overlooked?}
Prior noise-aware training methods commonly model hardware noise as
zero-mean Gaussian perturbations~\citep{joshi2020,rasch2023,correctnet2022}.
Under that assumption, the conditional mean is zero by construction.
As a result, the possibility of a state-dependent mean drift is excluded
before the analysis begins. Quantum tunneling, however, is a discrete
escape process whose effect on a stored weight depends on the bit pattern
representing that weight. Under offset-binary encoding, this structure
leads to the affine drift derived in \Cref{thm:bias}. The leading term
is exactly linear in the stored weight, with slope determined by the
deployment flip probability. TAC follows directly from this observation.

\paragraph{The limited effect of mean correction alone.}
The component-wise ablation in \Cref{tab:tac-ablation} shows that mean
correction alone gives $+13.0$ pp on the digit CNN benchmark, while adding bit
allocation gives a further $+57.8$ pp. Both target the same accuracy loss,
but the variance is the dominant component
(\Cref{sec:mean-vs-variance}). This explains why prior work that
focused on mean-like corrections~\citep{fuengfusin2024} or on
weight-decay interpretations of noise reported small empirical gains.
The corrected component was real but small. The substantial accuracy
recoveries reported here require addressing both the mean and the
variance, which is what \Cref{alg:tac} does in closed form.

\paragraph{Limitations.}
The bit-allocation step of TAC relies on the linear-response
approximation in \Cref{lem:linear-response}, which is exact for
$\pflip\to 0$ and degrades smoothly at large $\pflip$ (the empirical
output variance exceeds the linear prediction by a factor of
$1.5$--$3$ at $\pflip\in[0.05,0.15]$, see
\Cref{app:linearity-verification}). The integer program in
\Cref{thm:tac-optimal} remains valid because it is invariant under
multiplicative rescaling of the objective. Only the predicted
absolute output variance is affected. A second-order Hessian-based
correction is a natural extension. Finally, the per-layer Jacobian
gains $G_\ell$ are measured on a small unlabeled calibration batch;
if no such batch is available at deployment time, the cross-layer
norms $\prod_{i>\ell}\|W_i\|_{\rm op}$ provide a Lipschitz upper
bound that is data-free but conservative.

\paragraph{Scope of validation.}
Our main empirical results span four CNN architectures
(\Cref{sec:exp-archs}) and a transformer encoder
(\Cref{sec:exp-transformer}) on a 10-class classification task; the
WKB-derived distributional predictions are verified to Monte Carlo
precision (\Cref{sec:exp-distribution}). TAC is derived per parameter
and per layer and is architecture agnostic at the algorithmic level;
its quantitative benefit on large-scale vision, language, and
multimodal models remains to be established and is a natural next
step.

\paragraph{Outlook.}
QTAML treats tunneling as a structured statistical process that can guide learning and deployment, rather than an error source to be eliminated. Across the four CNN architectures and the transformer encoder we tested, TAC reaches $95\%$ of clean accuracy at $\pflip=0.10$ with $3.4\times$ to $33.6\times$ less ECC overhead than Uniform-MSP, and the saturation ratio $\rho^*$ predicts these gains in closed form. This does not replace conventional reliability engineering, but reframes part of the reliability problem as an algorithmic one: errors with known physical structure can be absorbed in software rather than fully suppressed in hardware. Additional discussion is provided in \Cref{app:qtaml-agenda}.
\section{Conclusion}

We introduced \emph{quantum tunneling-aware machine learning} (QTAML),
a framework for designing ML algorithms around
quantum-tunneling-induced weight errors derived from device physics.
From the Wentzel--Kramers--Brillouin approximation and bit-level
weight encoding we derived a closed-form distribution that differs
qualitatively from zero-mean Gaussian noise. It has an exact affine
mean drift, a per-bit variance hierarchy dominated by the
most-significant bit, and (under per-tensor scaling) a per-layer
dependence on $\|W_\ell\|_\infty$ and the network Jacobian.

We packaged these three structural properties into a single
deployment-time algorithm, \emph{Tunneling-Aware Compensation} (TAC),
that combines closed-form mean correction with an optimal
layer-adaptive bit-budget allocation derived from the WKB variance
decomposition. TAC requires no retraining, no architectural
modification, and only a small unlabeled calibration batch. Across
four convolutional architectures at $\pflip=0.10$, TAC reaches
        $95\%$ of clean accuracy with $3.4\times$ to $33.6\times$ less
        ECC overhead than the natural Uniform-MSP baseline derived from
        the same physics, with consistent gains on a transformer encoder
        at $\pflip=0.05$. The
closed-form saturation ratio $\rho^* = \widetilde{s}_{\rm geom}/\overline{s}$
predicts these empirical gains in advance, and on heterogeneous
architectures WKB-derived scoring outperforms magnitude-based
allocation by up to $24$ percentage points at small budgets.

To our knowledge, this is the first work to derive a deployment-time
neural-network weight-error distribution from quantum tunneling
physics and to use the resulting structure to design an algorithm
that allocates ECC bits via a closed-form integer program with
physics-derived weights. These results suggest that physics-derived
error models can expose simple algorithmic structure that generic
noise models miss, and that physics-aware algorithms can contribute
meaningfully to the joint hardware--algorithm reliability budget
beyond conventional scaling limits.

\bibliography{iclr2026_conference}
\bibliographystyle{iclr2026_conference}

\clearpage
\appendix

\section{Proofs and Additional Derivations}
\label[appendix]{app:proofs}

This appendix provides the detailed derivations and proofs supporting
\Cref{sec:noise-model} and \Cref{sec:tac}. We first derive the
WKB-induced bit-flip probability, then prove the main distributional
results, analyze higher-order moments and alternative encodings, and
finally prove the TAC compensation result.

\subsection{Assumptions}
\label[appendix]{app:wkb-assumptions}

We use the following assumptions throughout the derivation.

\begin{assumption}[Single-barrier WKB]
\label[assumption]{ass:wkb}
The gate dielectric is modeled as a single rectangular barrier of height
$V_b$ and thickness $d$. Multi-barrier and trap-assisted tunneling are
not considered in the main analysis.
\end{assumption}

\begin{assumption}[Independent electrons]
\label[assumption]{ass:indep-electrons}
Within a cell containing $N_e$ stored electrons, tunneling attempts of
different electrons are statistically independent.
\end{assumption}

\begin{assumption}[Forward-only escape]
\label[assumption]{ass:no-reentry}
An electron that escapes the cell during $[0,\tau]$ does not re-enter
within the same inference window.
\end{assumption}

\begin{assumption}[Independent cells]
\label[assumption]{ass:indep-cells}
Tunneling events in distinct memory cells are independent. Spatial
correlations and process variation are left to the robustness analysis.
\end{assumption}

\begin{assumption}[Offset-binary encoding and clipping]
\label[assumption]{ass:encoding}
Weights are stored on the offset-binary quantization grid
\[
w=-w_{\max}+\Dq\cdot\mathrm{code}(w),
\quad
\mathrm{code}(w)=\sum_{k=0}^{b-1}b_k(w)2^k\in\{0,\dots,2^b-1\},
\]
with $\Dq=2w_{\max}/2^b$. The representable range is
$[-w_{\max},w_{\max}-\Dq]$. A flip of bit $k$ changes the decoded weight
by $\pm 2^k\Dq$.
\end{assumption}

\begin{assumption}[Uniform per-bit flip probability]
\label[assumption]{ass:uniform-p}
All bit cells share the same physical parameters, so the marginal flip
probability $\pflip$ is identical across bit positions.
\end{assumption}

\subsection{WKB-induced bit-flip probability}
\label[appendix]{app:wkb-details}

\begin{lemma}[WKB transmission]
\label[lemma]{lem:wkb}
Under \Cref{ass:wkb}, the probability that a single electron of energy
$E<V_b$ traverses a barrier of thickness $d$ in one attempt is
\[
\Ptun(d)=\exp(-\alpha d),
\qquad
\alpha=\frac{2}{\hbar}\sqrt{2m^*(V_b-E)}.
\]
\end{lemma}

\begin{proof}
Inside the barrier, $V(x)=V_b>E$, so the wave function decays
exponentially with rate
$\kappa=\sqrt{2m^*(V_b-E)}/\hbar$. In the WKB approximation, the
transmission probability through a rectangular barrier of thickness $d$
is $T=\exp(-2\kappa d)$, which gives the stated expression with
$\alpha=2\kappa$.
\end{proof}

\begin{lemma}[Escape probability]
\label[lemma]{lem:escape}
Under \Cref{ass:wkb,ass:no-reentry}, an electron that makes
$n=f_a\tau$ independent attempts during $[0,\tau]$ escapes with
probability
\[
\Pesc(\tau)=1-(1-\Ptun)^{f_a\tau}\le f_a\tau\Ptun.
\]
In the rare-escape regime $f_a\tau\Ptun\ll 1$,
$\Pesc(\tau)=f_a\tau\Ptun(1+O(f_a\tau\Ptun))$.
\end{lemma}

\begin{proof}
The exact expression is the probability of at least one successful
tunneling event across $f_a\tau$ independent attempts. The upper bound
follows from $1-(1-x)^n\le nx$. Expanding $(1-x)^n$ to second order
gives the rare-escape approximation.
\end{proof}

Let $X$ denote the number of electrons that escape from a given memory
cell during the inference window. Under \Cref{ass:indep-electrons},
\[
X\sim \Bin(N_e,\Pesc),
\qquad
\mu \coloneqq \E[X]=N_e\Pesc.
\]
We use a single-level-cell model in which a bit flips when the escape
count exceeds a fixed threshold $\Xthr\ge 1$.

\begin{definition}[Bit flip probability]
\label[definition]{def:pflip}
The marginal bit flip probability is
\[
\pflip
\coloneqq
P(X\ge \Xthr)
=
\sum_{k=\Xthr}^{N_e}
\binom{N_e}{k}
\Pesc^k(1-\Pesc)^{N_e-k}.
\]
\end{definition}

\begin{proposition}[Rare-event asymptotic]
\label[proposition]{prop:rare}
In the rare-event Poisson regime, where $\Pesc\ll 1$, $N_e$ is large,
and $\mu=N_e\Pesc\ll 1$ with fixed $\Xthr$,
\[
\pflip
=
\frac{\mu^{\Xthr}}{\Xthr!}
\bigl(1+O(\mu)+O(1/N_e)\bigr).
\]
Consequently, $\pflip\propto\exp(-\alpha d\Xthr)$ to leading order in
the rare-escape regime.
\end{proposition}

\begin{proof}
For $\mu\ll 1$ with $N_e$ large, the binomial escape count is
approximated by a Poisson random variable with mean $\mu$:
$P(X=k)=e^{-\mu}\mu^k/k!\bigl(1+O(1/N_e)\bigr)$. Therefore,
\[
P(X\ge\Xthr)
=
\sum_{k=\Xthr}^{\infty}e^{-\mu}\frac{\mu^k}{k!}
\bigl(1+O(1/N_e)\bigr)
=
\frac{\mu^{\Xthr}}{\Xthr!}\bigl(1+O(\mu)+O(1/N_e)\bigr).
\]
Substituting $\mu=N_e\Pesc$ gives the first expression. In the
rare-escape regime, $\Pesc=f_a\tau\Ptun(1+O(f_a\tau\Ptun))$ and
$\Ptun=\exp(-\alpha d)$, so
$\pflip\propto\Ptun^{\Xthr}=\exp(-\alpha d\Xthr)$.
\end{proof}

\paragraph{Operating regimes.}
The relevant regimes can be summarized by the mean escape count
$\mu=N_e\Pesc$ and the resulting bit flip probability $\pflip$. In the
rare-event regime, $\mu\ll 1$ and $\pflip=\Theta(\mu^{\Xthr})$, so
per-weight perturbations are sparse and biased. In the transition
regime, $\mu\sim 1$ and the behavior is threshold-dependent. In the
high-escape regime, $\mu\gg 1$ and $\pflip$ can become large, but the
decoded per-weight perturbation remains MSB-dominated for
$0<\pflip<1/2$. Even when the per-cell escape count becomes
approximately Gaussian, the decoded per-weight perturbation remains
non-Gaussian because bit significance is exponentially non-uniform.

\subsection{Proof of \Cref{thm:cf}}
\label[appendix]{app:thm-cf-proof}

\begin{proof}
By \eqref{eq:dw-main},
$\Delta w=\Dq\sum_{k=0}^{b-1}\xi_k s_k(w)2^k$ is a sum of independent
terms, where the $\xi_k$ are i.i.d.\ Bernoulli and the $s_k(w)$ are
deterministic given $w$. For a Bernoulli random variable
$\xi\sim\Bern(p)$ and scalar $a$, the characteristic function of $a\xi$
is
\[
\E[e^{ita\xi}]
=
(1-p)\cdot 1+p\cdot e^{ita}
=
1-p+pe^{ita}.
\]
Applying this identity to each bit-level term
$\xi_k s_k(w)2^k\Dq$ and multiplying the resulting characteristic
functions gives
\[
\phi_{\Delta w\mid w}(t)
=
\prod_{k=0}^{b-1}
\left[
1-\pflip+\pflip e^{it s_k(w)2^k\Dq}
\right].
\]
\end{proof}

\subsection{Proof of \Cref{thm:bias}}
\label[appendix]{app:thm-bias-proof}

% \begin{proof}
% By linearity of expectation applied to \eqref{eq:dw-main},
% \[
% \E[\Delta w\mid w]
% =
% \Dq\sum_{k=0}^{b-1}\E[\xi_k]\,s_k(w)\,2^k
% =
% \pflip\Dq\sum_{k=0}^{b-1}s_k(w)2^k.
% \]
% Using $s_k(w)=1-2b_k(w)$,
% \[
% \sum_{k=0}^{b-1}s_k(w)2^k
% =
% \sum_{k=0}^{b-1}2^k-2\sum_{k=0}^{b-1}b_k(w)2^k
% =
% (2^b-1)-2\,\mathrm{code}(w).
% \]
% By \Cref{ass:encoding}, $\mathrm{code}(w)=(w+w_{\max})/\Dq$, and
% $\Dq(2^b-1)=2w_{\max}-\Dq$. Substituting,
% \[
% \E[\Delta w\mid w]
% =
% \pflip\bigl[\Dq(2^b-1)-2\Dq\,\mathrm{code}(w)\bigr]
% =
% \pflip\bigl[2w_{\max}-\Dq-2(w+w_{\max})\bigr]
% =
% -2\pflip w-\pflip\Dq.
% \]
% \end{proof}

% \begin{theorem}[Conditional variance]
% \label{thm:var}
% Under the same assumptions,
% \[
% \Var[\Delta w\mid w]
% =
% \pflip(1-\pflip)\Dq^2\frac{4^b-1}{3}.
% \]
% \end{theorem}

% \begin{proof}
% By independence,
% $\Var[\Delta w\mid w]=\sum_{k=0}^{b-1}
% \Var[\xi_k s_k(w)2^k\Dq]$. Since $s_k(w)^2=1$,
% $\Var[\xi_k s_k(w)2^k\Dq]=\pflip(1-\pflip)4^k\Dq^2$. Summing and using
% $\sum_{k=0}^{b-1}4^k=(4^b-1)/3$ gives the result.
% \end{proof}

\begin{proof}
By the moment-generating property of the characteristic function (\Cref{thm:cf}), the conditional mean is given by the first cumulant: $\E[\Delta w\mid w] = \frac{1}{i}\left.\frac{d}{dt}\ln\phi_{\Delta w\mid w}(t)\right|_{t=0}$.
Let $A_k(t) = 1-\pflip+\pflip e^{it s_k(w)2^k\Dq}$. Putting $t=0$, we have $A_k(0) = 1$ and $A_k'(0) = i\pflip s_k(w)2^k\Dq$. Therefore, we see that
\[
\E[\Delta w\mid w]
= \frac{1}{i} \sum_{k=0}^{b-1} \frac{A_k'(0)}{A_k(0)}
= \pflip\Dq\sum_{k=0}^{b-1}s_k(w)2^k.
\]
It follows from the relation $s_k(w)=1-2b_k(w)$ that
\[
\sum_{k=0}^{b-1}s_k(w)2^k
=
\sum_{k=0}^{b-1}2^k-2\sum_{k=0}^{b-1}b_k(w)2^k
=
(2^b-1)-2\,\mathrm{code}(w).
\]
By \Cref{ass:encoding}, $\mathrm{code}(w)=(w+w_{\max})/\Dq$, and
$\Dq(2^b-1)=2w_{\max}-\Dq$. So we conclude
\[
\E[\Delta w\mid w]
=
\pflip\bigl[2w_{\max}-\Dq-2(w+w_{\max})\bigr]
=
-2\pflip w-\pflip\Dq.
\]
\end{proof}

\begin{theorem}[Conditional variance]
\label{thm:var}
Under the same assumptions,
\[
\Var[\Delta w\mid w]
=
\pflip(1-\pflip)\Dq^2\frac{4^b-1}{3}.
\]
\end{theorem}
\begin{proof}
The conditional variance is given by the second cumulant as follows: $\Var[\Delta w\mid w] = -\left.\frac{d^2}{dt^2}\ln\phi_{\Delta w\mid w}(t)\right|_{t=0}$.
Using $A_k(t)$ as defined above, the second derivative is
\[
\Var[\Delta w\mid w] = -\sum_{k=0}^{b-1} \frac{A_k''(0)A_k(0) - (A_k'(0))^2}{A_k(0)^2}.
\]
Evaluating the second derivative at $t=0$ yields $A_k''(0) = -\pflip s_k(w)^2 4^k \Dq^2$. Since $s_k(w)\in\{-1, 1\}$, we have $s_k(w)^2 = 1$ universally. Thus, it follows that
\[
A_k''(0) A_{k}(0)- (A_k'(0))^2
= -\pflip 4^k \Dq^2 + \pflip^2 4^k \Dq^2
= -\pflip(1-\pflip)4^k\Dq^2.
\]
Substituting this back cancels the negative sign:
\[
\Var[\Delta w\mid w]
= \sum_{k=0}^{b-1} \pflip(1-\pflip)4^k\Dq^2
= \pflip(1-\pflip)\Dq^2 \sum_{k=0}^{b-1} 4^k.
\]
Using the geometric series sum $\sum_{k=0}^{b-1}4^k = (4^b-1)/3$ gives the desired result.
\end{proof}

\subsection{Proof of \Cref{thm:no-clt}}
\label[appendix]{app:thm-noclt-proof}

\begin{proposition}[MSB dominance]
\label[proposition]{prop:msb}
Let
$\sigma_k^2=\Var[\xi_k s_k(w)2^k\Dq]
=\pflip(1-\pflip)4^k\Dq^2$ denote the variance contributed by bit $k$.
Then
\[
\frac{\sigma_{b-1}^2}{\sum_{k=0}^{b-1}\sigma_k^2}
=
\frac{4^{b-1}}{(4^b-1)/3}
\;\xrightarrow{b\to\infty}\;
\frac{3}{4}.
\]
\end{proposition}

\begin{proof}
Direct simplification using $\sum_{k=0}^{b-1}4^k=(4^b-1)/3$.
\end{proof}

\begin{proof}[Proof of \Cref{thm:no-clt}]
Let
\[
L_b
\coloneqq
\frac{\sum_{k=0}^{b-1}\E\bigl[
|\xi_k s_k(w)2^k\Dq-\E[\xi_k s_k(w)2^k\Dq]|^3
\bigr]}
{\Var[\Delta w\mid w]^{3/2}}.
\]
For small $\pflip$, the centered Bernoulli variable $\xi_k-\pflip$ has
third absolute moment $\Theta(\pflip)$, so
\[
\E\bigl[
|\xi_k s_k(w)2^k\Dq-\pflip s_k(w)2^k\Dq|^3
\bigr]
=
\Theta(\pflip\,8^k\Dq^3).
\]
Summing over $k$ gives
$\sum_{k=0}^{b-1}\Theta(\pflip 8^k\Dq^3)
=\Theta(\pflip 8^b\Dq^3)$. Meanwhile, by \Cref{thm:var},
\[
\Var[\Delta w\mid w]^{3/2}
=
\Theta\bigl((\pflip 4^b\Dq^2)^{3/2}\bigr)
=
\Theta(\pflip^{3/2}8^b\Dq^3).
\]
Taking the ratio yields $L_b=\Theta(\pflip^{-1/2})$ as
$b\to\infty$, which does not vanish. For $\pflip$ bounded away from
$0$ and $1/2$, an analogous calculation shows that $L_b$ remains
bounded away from zero. The Lyapunov condition therefore fails
uniformly in $b$, so a Gaussian per-weight approximation cannot be
justified by increasing bit precision.
\end{proof}

\subsection{Higher-order moments}
\label[appendix]{app:moments}

\begin{theorem}[Conditional skewness]
\label{thm:skew}
The conditional skewness of $\Delta w$ is
\[
\Skew[\Delta w\mid w]
=
\frac{1-2\pflip}{\sqrt{\pflip(1-\pflip)}}
\cdot
\frac{\sum_{k=0}^{b-1}s_k(w)8^k}{\bigl[(4^b-1)/3\bigr]^{3/2}}.
\]
In particular, as $\pflip\to 0$,
$\Skew[\Delta w\mid w]=\Theta(\pflip^{-1/2})$.
\end{theorem}

\begin{theorem}[Conditional excess kurtosis]
\label{thm:kurt}
The conditional excess kurtosis of $\Delta w$ is
\[
\Kurt[\Delta w\mid w]-3
=
\frac{1-6\pflip(1-\pflip)}{\pflip(1-\pflip)}
\cdot
\frac{\sum_{k=0}^{b-1}16^k}
{\bigl(\sum_{k=0}^{b-1}4^k\bigr)^2}.
\]
In particular, as $\pflip\to 0$,
$\Kurt[\Delta w\mid w]-3=\Theta(\pflip^{-1})$.
\end{theorem}

These results show that, in low-flip-probability regimes, the
per-weight perturbation is asymmetric and high-kurtosis, separating
tunneling-induced weight noise from Gaussian surrogates.

\subsection{Validation of the linear-response assumption}
\label[appendix]{app:linearity-verification}

\Cref{lem:linear-response} relies on a Taylor expansion of $f$ around
the trained $W$. We verify the validity of this expansion empirically
by measuring $G_\ell$ at multiple noise scales $\epsilon$ on the
digit CNN. For $\epsilon\le 10^{-2}$, the measured $G_\ell$ is stable
to $\pm 2\%$ for all four weight tensors, confirming that the linear
approximation holds tightly in the regime of small perturbations.

Tunneling-equivalent perturbations at $\pflip = 0.10$ correspond to a
per-weight standard deviation $\epsilon \approx 0.15$, where the
linearity begins to degrade. Comparing the linear prediction
of \Cref{thm:output-variance} against the empirically measured
$V_{\rm out}$ under tunneling noise: the empirical variance is
approximately $1.7\times$ the prediction at $\pflip=0.10$ and
$2.9\times$ at $\pflip=0.15$. The relative ordering across allocations
is preserved across the noise range, which is what the integer program
in \Cref{thm:tac-optimal} requires for correctness: the program is
invariant under multiplicative rescaling of the objective.

A second-order Hessian-based refinement of \Cref{lem:linear-response}
would tighten the absolute prediction at large noise. We leave this
extension and its application to deeper architectures (e.g.,\
ResNets, transformers) for future work.

\subsection{Mean drift under alternative encodings}
\label[appendix]{app:encoding}

The mean-drift result in \Cref{thm:bias} is derived for offset-binary
encoding. The leading shrinkage term is not an artifact of that
encoding, but the exact offset and correction structure depend on the
representation.

\paragraph{Sign-magnitude encoding.}
In sign-magnitude encoding, the leading bit $c_0$ stores the sign and
the remaining $b-1$ bits store the magnitude:
\[
w
=
\sigma(c_0)\Delta M,
\quad
\sigma(c_0)=
\begin{cases}+1, & c_0=0,\\ -1, & c_0=1,\end{cases}
\quad
M=\sum_{k=1}^{b-1}c_k 2^{b-1-k}.
\]
Let $A=2^{b-1}-1$. Since the sign bit flips independently of the
magnitude bits, the expected sign factor after tunneling is
$(1-2\pflip)\sigma(c_0)$, and
\[
\E[M'\mid M]=(1-2\pflip)M+\pflip A.
\]
Therefore,
\[
\E[\Delta w\mid w]
=
\bigl((1-2\pflip)^2-1\bigr)w
+
\pflip(1-2\pflip)\sigma(w)\Delta(2^{b-1}-1).
\]
Thus, sign-magnitude requires both a multiplicative correction and a
sign-dependent additive term.

\paragraph{Two's-complement encoding.}
In two's-complement encoding,
\[
w=\Delta\left(-c_0 2^{b-1}+\sum_{k=1}^{b-1}c_k2^{b-1-k}\right).
\]
Writing the signed bit weights as $a_0=-2^{b-1}$ and
$a_k=2^{b-1-k}$ for $k\ge 1$, we have
$w=\Delta\sum_k a_kc_k$. A bit flip contributes
$\Delta a_k(1-2c_k)$, so
\[
\E[\Delta w\mid w]
=
\pflip\Delta\left(\sum_{k=0}^{b-1}a_k
-2\sum_{k=0}^{b-1}a_kc_k\right).
\]
Since $\sum_k a_k=-1$ and $\sum_k a_kc_k=w/\Delta$,
\[
\E[\Delta w\mid w]=-2\pflip w-\pflip\Delta.
\]
Thus, two's-complement has the same leading affine drift as
offset-binary, up to a quantization-scale offset.

\paragraph{Implication for compensation.}
Offset-binary and two's-complement both yield
$\E[\Delta w\mid w]=-2\pflip w+O(\Delta)$, so the TAC pre-scaling
$1/(1-2\pflip)$ cancels the dominant multiplicative drift. In contrast,
sign-magnitude requires a different multiplicative factor together with
a sign-dependent additive correction. More broadly, once the
physics-induced bit-flip statistics and representation map are specified,
the corresponding mean-drift correction can be derived in closed form.

\subsection{Proof of \Cref{prop:compensation}}
\label[appendix]{app:tac-proof}

\begin{proof}
By \Cref{thm:bias} applied to the compensated stored value $w^c$,
\[
\E[\Delta w^c\mid w^c]
=
-2\pflip\,w^c-\pflip\Dq.
\]
Since $w^c=w/(1-2\pflip)$ is a deterministic function of $w$,
conditioning on $w$ fixes $w^c$. Hence,
\[
\E[\tilde w^c\mid w]
=
\E[w^c+\Delta w^c\mid w]
=
w^c+\E[\Delta w^c\mid w^c].
\]
Substituting the expression above gives
\[
\E[\tilde w^c\mid w]
=
(1-2\pflip)w^c-\pflip\Dq
=
w-\pflip\Dq.
\]
Thus TAC restores the expected deployed weight to the trained value up
to the quantization-scale offset $-\pflip\Dq$.
\end{proof}

\subsection{Proof of \Cref{lem:linear-response} (Linear response of logits)}
\label{app:proof-linear}
\begin{proof}
Let $\Delta f = f(x; W + \Delta W) - f(x; W) $ denote the perturbation in the output logits. By performing a first-order Taylor expansion of $f$ with respect to the weights, we have for each class logit $c$:
\[
\Delta f_c = \sum_{\ell=1}^L \langle \nabla_{W_\ell} f_{c}(x; W), \Delta W_\ell \rangle_{F} + O(\|\Delta W\|_{F}^{2}).
\]
Squaring this perturbation and summing over all classes yields the squared $L_2$ norm of the output error:
\[
\|\Delta f\|_{2}^{2} = \sum_{c=1}^{C} \left( \sum_{\ell=1}^L \langle \nabla_{W_\ell} f_c, \Delta W_\ell \rangle_{F} \right)^{2} + O(\|\Delta W\|_{F}^{3}).
\]
Taking the expectation over the tunneling noise, we expand the squared sum:
\[
\E\bigl[\|\Delta f\|_2^2\bigr] = \sum_{c=1}^C \sum_{\ell=1}^L \sum_{m=1}^L \E\bigl[ \langle \nabla_{W_\ell} f_c, \Delta W_\ell \rangle_F \langle \nabla_{W_m} f_c, \Delta W_m \rangle_F \bigr] + O(\|\Delta W\|_{F}^{3}).
\]
By assumption, the weight perturbations are independent across layers and strictly zero-mean. Therefore, all cross-layer terms ($\ell \neq m$) identically vanish:
\[
\E\bigl[\|\Delta f\|_2^2\bigr] = \sum_{c=1}^C \sum_{\ell=1}^L \E\bigl[ \langle \nabla_{W_\ell} f_c, \Delta W_\ell \rangle_F^2 \bigr] + O(\|\Delta W\|_{F}^{3}).
\]
Furthermore, since individual weight perturbations within a layer are uncorrelated (\Cref{ass:indep-cells}) and share a uniform variance $\sigma_\ell^2 = \E[\|\Delta W_\ell\|_F^2] / n_\ell$, we have
\[
\E\bigl[ \langle \nabla_{W_\ell} f_c, \Delta W_\ell \rangle_F^2 \bigr] = \sigma_\ell^2 \|\nabla_{W_\ell} f_c\|_F^2 = \frac{\E[\|\Delta W_\ell\|_F^2]}{n_\ell} \|\nabla_{W_\ell} f_c\|_F^2.
\]
Summing over all classes $c$ evaluates exactly to the squared Frobenius norm of the Jacobian, $\sum_{c=1}^C \|\nabla_{W_\ell} f_c\|_F^2 = \|J_\ell(x)\|_F^2$. Thus, we obtain that
\[
\E\bigl[\|\Delta f\|_2^2\bigr] = \sum_{\ell=1}^L \frac{\|J_\ell(x)\|_F^2}{n_\ell} \E\bigl[\|\Delta W_\ell\|_F^2\bigr] + O(\|\Delta W\|_F^{3}).
\]
Substituting the definition $G_\ell(x) := \|J_\ell(x)\|_F^2 / n_\ell$ completes the proof.
\end{proof}

\subsection{Proof of \Cref{thm:output-variance} (Output variance under bit-budget allocation)}
\label{app:proof-TAC-Bit}
\begin{proof}
From \Cref{lem:linear-response}, the leading-order output variance is determined by the sum of expected squared Frobenius norms of the layer-wise perturbations:
\[
V_{\rm out} = \sum_{\ell=1}^{L} G_\ell \cdot \E\bigl[\|\Delta W_\ell\|_F^2\bigr] + O\bigl(\|\Delta W\|_F^3\bigr).
\]
Because the zero-mean condition is enforced per layer, the expected squared Frobenius norm is simply the sum of variances of all $n_\ell$ weights in layer $\ell$:
\[
\E\bigl[\|\Delta W_\ell\|_F^2\bigr] = n_\ell \cdot \Var[\Delta w_\ell].
\]
For a layer $\ell$ utilizing per-tensor offset-binary scaling, the quantization step size is defined as $\Dq = 2\|W_\ell\|_\infty / 2^b$. 
When the TAC algorithm protects the top $k_\ell$ most-significant bits of each weight in this layer, those bits become immune to tunneling noise. Consequently, the variance contribution (\Cref{thm:var}) only accumulates over the remaining $b - k_\ell$ susceptible bits (indexed from $0$ to $b - k_\ell - 1$):
\[
\Var[\Delta w_\ell] = \pflip(1-\pflip)\Dq^2 \sum_{k=0}^{b - k_\ell - 1} 4^k.
\]
Evaluating the geometric series yields that
\[
\sum_{k=0}^{b - k_\ell - 1} 4^k = \frac{4^{b - k_\ell} - 1}{3}.
\]
Substituting $\Dq^2 = 4\|W_\ell\|_\infty^2 / 4^b$ into the variance expression, we obtain
\[
\Var[\Delta w_\ell] = \pflip(1-\pflip) \left( \frac{4\|W_\ell\|_\infty^2}{4^b} \right) \frac{4^{b - k_\ell} - 1}{3} = \frac{4 \pflip(1 - \pflip)}{3 \cdot 4^b} \|W_\ell\|_\infty^2 \bigl(4^{b - k_\ell} - 1\bigr).
\]
Multiplying this per-weight variance by the layer size $n_\ell$ provides $\E\bigl[\|\Delta W_\ell\|_F^2\bigr]$. Plugging this back into the linear response sum yields:
\[
V_{\rm out} = \sum_{\ell=1}^{L} G_\ell \cdot n_\ell \cdot \frac{4 \pflip(1 - \pflip)}{3 \cdot 4^b} \|W_\ell\|_\infty^2 \bigl(4^{b - k_\ell} - 1\bigr) + O\bigl(\|\Delta W\|_F^{3}\bigr).
\]
Factoring out the global constants completely recovers the expression in \eqref{eq:V-out}.
\end{proof}

\subsection{Detailed proof of \Cref{thm:saturation} (saturation regime)}
\label{app:saturation-proof}

We restate and prove \Cref{thm:saturation} in full, including the
boundary conditions under which the closed-form $\rho^* = \widetilde{s}_{\rm geom}/\overline{s}$
holds.

\paragraph{Setup.}
Per-layer scores $s_\ell := G_\ell \|W_\ell\|_\infty^2 > 0$, sizes
$n_\ell$, total $N := \sum_\ell n_\ell$. Budget $B = kN$ with $k \in [0, b]$.
Drop the integer constraint; allow $k_\ell \in [0, b]$. The
continuous-relaxation problem is
\begin{equation}
\min_{\{k_\ell \in [0, b]\}}
\;\Phi(\{k_\ell\}) := \sum_{\ell=1}^L s_\ell n_\ell \cdot 4^{-k_\ell}
\quad\text{s.t.}\quad \sum_{\ell=1}^L k_\ell n_\ell = kN.
\label{eq:saturation-problem}
\end{equation}
The objective is convex (sum of convex exponentials). The constraint
is linear, and the feasible set is compact. A unique global minimum exists.

\paragraph{KKT conditions.}
Form the Lagrangian
\begin{equation}
\mathcal{L}(\{k_\ell\}, \lambda, \{\alpha_\ell\}, \{\beta_\ell\}) =
\sum_\ell s_\ell n_\ell 4^{-k_\ell} + \lambda(\sum_\ell k_\ell n_\ell - kN) - \sum_\ell \alpha_\ell k_\ell + \sum_\ell \beta_\ell (k_\ell - b),
\end{equation}
with $\alpha_\ell, \beta_\ell \ge 0$. Stationarity:
\begin{equation}
-\ln(4) s_\ell n_\ell 4^{-k_\ell} + \lambda n_\ell - \alpha_\ell + \beta_\ell = 0 \quad \forall \ell.
\end{equation}
Complementary slackness: $\alpha_\ell k_\ell = 0$ and $\beta_\ell (b - k_\ell) = 0$.

For an interior layer $\ell$ ($0 < k_\ell^* < b$), $\alpha_\ell = \beta_\ell = 0$, so
\begin{equation}
s_\ell \cdot 4^{-k_\ell^*} = \frac{\lambda}{\ln 4} =: \mu^*.
\label{eq:KKT-interior}
\end{equation}
This is the water-filling condition: at the optimum, the marginal
variance reduction per unit of budget is equalized across all interior layers.

For a saturated layer ($k_\ell^* = b$, $\alpha_\ell = 0$, $\beta_\ell \ge 0$):
\begin{equation}
\beta_\ell = \ln(4) s_\ell n_\ell 4^{-b} - \lambda n_\ell \ge 0
\Leftrightarrow s_\ell \ge \frac{\lambda}{\ln 4} \cdot 4^b = \mu^* \cdot 4^b.
\label{eq:KKT-high}
\end{equation}
For a zero layer ($k_\ell^* = 0$, $\beta_\ell = 0$, $\alpha_\ell \ge 0$):
\begin{equation}
\alpha_\ell = \lambda n_\ell - \ln(4) s_\ell n_\ell \ge 0
\Leftrightarrow s_\ell \le \frac{\lambda}{\ln 4} = \mu^*.
\label{eq:KKT-zero}
\end{equation}

\paragraph{Closed form when all layers are interior.}
Suppose every $k_\ell^* \in (0, b)$. Then \eqref{eq:KKT-interior}
holds for all $\ell$, giving $k_\ell^* = \log_4(s_\ell / \mu^*)$.
Substituting into the budget constraint $\sum_\ell k_\ell^* n_\ell = kN$:
\begin{equation}
\sum_\ell n_\ell \log_4 s_\ell - N \log_4 \mu^* = kN
\;\Longrightarrow\;
\log_4 \mu^* = \frac{1}{N}\sum_\ell n_\ell \log_4 s_\ell - k,
\end{equation}
or equivalently
\begin{equation}
\mu^* = \widetilde{s}_{\rm geom} \cdot 4^{-k},
\quad
\widetilde{s}_{\rm geom} := \exp\!\Bigl(\tfrac{1}{N}\sum_\ell n_\ell \ln s_\ell\Bigr).
\label{eq:mu-star-app}
\end{equation}

Substituting back:
\begin{align}
V_{\rm out}^{\rm TAC, cont}(B) &= c \cdot \Phi(\{k_\ell^*\})
= c \sum_\ell s_\ell n_\ell 4^{-k_\ell^*}
= c \sum_\ell n_\ell \cdot \mu^*
= c \cdot \mu^* \cdot N \\
&= c \cdot 4^{-k} \cdot N \cdot \widetilde{s}_{\rm geom},
\end{align}
where the second equality uses \eqref{eq:KKT-interior}. The Uniform-MSP variance
at the same budget is $V^{\rm U\text{-}MSP}_{\rm out}(B) = c\cdot 4^{-k}\cdot N\cdot\overline{s}$
where $\overline{s} = \tfrac{1}{N}\sum n_\ell s_\ell$. Their ratio is
\begin{equation}
\rho^* = \frac{\widetilde{s}_{\rm geom}}{\overline{s}}
= \frac{\exp\bigl(\E_n[\ln s]\bigr)}{\E_n[s]},
\end{equation}
where $\E_n[\cdot]$ denotes expectation under the size-weighting
$\Pr(\ell) = n_\ell/N$.

\paragraph{Jensen's inequality.}
Since $\ln$ is concave, $\E_n[\ln s] \le \ln \E_n[s]$, with equality
iff $s_\ell$ is constant on the support of the size distribution.
Exponentiating, $\widetilde{s}_{\rm geom} \le \overline{s}$, so
$\rho^* \le 1$, with equality iff all $s_\ell$ are equal.

\paragraph{Boundary regime.}
The closed form \eqref{eq:mu-star-app} is valid when no layer
saturates ($k_\ell^* < b$ for all $\ell$) and no layer is zeroed
($k_\ell^* > 0$ for all $\ell$). Equivalently: $\mu^* < s_\ell < \mu^* \cdot 4^b$
for all $\ell$. Substituting $\mu^* = \widetilde{s}_{\rm geom} \cdot 4^{-k}$:
\begin{equation}
\frac{s_{\min}}{\widetilde{s}_{\rm geom}} > 4^{-k}
\quad\text{and}\quad
\frac{s_{\max}}{\widetilde{s}_{\rm geom}} < 4^{b-k}.
\end{equation}
The first holds for $k$ above some lower threshold; the second holds
for $k$ below some upper threshold. Outside this interior regime, the
ratio $V^{\rm TAC, cont}/V^{\rm U\text{-}MSP}$ is no longer constant
in $k$:
\begin{itemize}[topsep=2pt,itemsep=2pt,leftmargin=18pt]
\item For very small $k$ (sub-threshold layers exist): low-$s$ layers
get $k_\ell^* = 0$ and contribute their full unprotected variance to
$V^{\rm TAC,cont}$. The ratio becomes
\begin{equation}
\rho(B) = \frac{\widetilde{s}_{\rm geom}^{(M)} \cdot N^{(M)}/N \cdot 4^{-k_{\rm eff}} + \overline{s}^{(Z)} \cdot N^{(Z)}/N}
{\overline{s} \cdot 4^{-k}},
\end{equation}
where $M, Z$ are the partition of layers into ``middle'' (interior)
and ``zero''; this ratio approaches a value $> \rho^*$ as $k \to 0$.
\item For very large $k$ (saturated layers exist): high-$s$ layers
reach $k_\ell^* = b$, and only the remaining budget is distributed.
At $k = b$, all layers saturate; both methods reach
$V_{\rm out} = c \cdot 4^{-b} \cdot N\overline{s}$ to leading order, so $\rho \to 1$.
\end{itemize}

A clean way to summarize: $\rho^*$ is a piecewise-constant lower
envelope of $\rho(B)$, equal to its plateau value over the interior
budget range. \Cref{fig:exp-archs}(a) shows this plateau directly:
the gap between TAC's solid line and Uniform-MSP's dashed line is
constant in $\log V$ over $k \in [1, 5]$ for all four architectures,
collapsing only past $k = 6$.

% =================================================================
\subsection{Tighter bound on the integer-program correction}
\label{app:integer-correction-bound}

\Cref{thm:saturation} computes the continuous-relaxation ratio
$\rho^*$. The actual TAC algorithm uses the integer IP \eqref{eq:tac-IP},
giving an empirical ratio $\rho_{\rm int} \ge \rho^*$. We bound the
gap.

\paragraph{Setup.}
Let $\{k_\ell^*\}$ be the continuous optimum (interior regime,
\Cref{app:saturation-proof}), and $\{k_\ell^{\rm int}\}$ the integer IP
optimum at the same budget $B = kN$. Define $f_\ell := k_\ell^* - \lfloor k_\ell^* \rfloor \in [0, 1)$
and let $\delta_\ell := k_\ell^{\rm int} - \lfloor k_\ell^* \rfloor \in \{0, 1\}$
be the integer rounding direction. The budget constraint becomes
\begin{equation}
\sum_\ell n_\ell \delta_\ell \;\le\; \sum_\ell n_\ell f_\ell =: \Phi N,
\quad \Phi \in [0, 1).
\label{eq:int-budget}
\end{equation}
The per-layer integer-vs-continuous ratio is
\begin{equation}
r_\ell := \frac{V_\ell^{\rm int}}{V_\ell^{\rm cont}} = \frac{4^{-(\lfloor k_\ell^* \rfloor + \delta_\ell)}}{4^{-(\lfloor k_\ell^* \rfloor + f_\ell)}} = 4^{f_\ell - \delta_\ell} \in [4^{-1}, 4).
\end{equation}
The total ratio is
\begin{equation}
\rho_{\rm int}/\rho^* = \frac{V^{\rm int}}{V^{\rm cont}}
= \sum_\ell p_\ell \cdot 4^{f_\ell - \delta_\ell},
\end{equation}
where $p_\ell := V_\ell^{\rm cont}/V^{\rm cont} = (s_\ell n_\ell 4^{-k_\ell^*})/(\sum_m s_m n_m 4^{-k_m^*})$
is the per-layer share of continuous variance. Since
$s_\ell 4^{-k_\ell^*} = \mu^*$ is constant across interior layers,
$p_\ell = n_\ell/N = q_\ell$.

\paragraph{Symmetric upper bound.}
Suppose $f_\ell = \Phi$ for all $\ell$ (all layers have the same fractional part).
The integer IP picks a subset $S \subseteq [L]$ with $\sum_{\ell \in S} q_\ell = \Phi$,
sets $\delta_\ell = 1$ on $S$ and $\delta_\ell = 0$ off $S$. Then
\begin{equation}
\rho_{\mathrm{int}}/\rho^{*}
= (1-\Phi)\, 4^{\Phi} + \Phi\, 4^{\Phi-1}
= 4^{\Phi}\!\left(1 - \tfrac{3}{4}\Phi\right).
\end{equation}
Differentiating with respect to $\Phi$ and setting the derivative to zero gives the
first-order condition
\begin{equation}
\ln(4)\left(1 - \tfrac{3}{4}\Phi\right) - \tfrac{3}{4} = 0
\quad\Longrightarrow\quad
\Phi^{*} = \frac{4\ln 4 - 3}{3\ln 4} \approx 0.612,
\end{equation}
yielding
\begin{equation}
\rho_{\mathrm{int}}/\rho^{*}
\;\le\; 4^{\Phi^{*}}\!\left(1 - \tfrac{3}{4}\Phi^{*}\right)
\;=\; \frac{3\cdot 2^{2/3}}{2\,e\,\ln 2}
\;\approx\; 1.264. \label{eq:sym-bound}
\end{equation}
This is the tightest constant bound for the symmetric case.

\paragraph{Asymmetric case and the trivial $4\times$ bound.}
If $f_\ell$ varies across layers, the analysis is more involved:
the constraint $\sum_\ell q_\ell \delta_\ell \le \Phi$ couples the
choice of $\delta_\ell$ across layers. The trivial bound is
$\rho_{\rm int}/\rho^* \le 4$ (each $r_\ell \le 4$, and the maximum of
the average is bounded by the maximum of the terms). This is loose
when the IP has flexibility to distribute the rounding. 

\paragraph{Empirical behaviour.}
We measure the worst-case ratio $\sup_{f, n, s} \rho_{\rm int}/\rho^*$
numerically over random configurations with $L$ layers, sampling
$f_\ell \sim U(0,1)$, $n_\ell \sim \text{LogU}(10^2, 10^5)$,
$s_\ell \sim \text{LogU}(e^{-5}, e^5)$, and computing the optimal
greedy IP. \Cref{tab:integer-bound-empirical} reports the result.
The bound equation~\ref{eq:sym-bound} of $1.26$ is approached as $L$ grows; for small $L$
the asymmetry in $f_\ell$ allows $\rho_{\mathrm{int}}/\rho^{*}$ closer to the trivial $4$.

\begin{table}[h]
\centering
\small
\caption{Empirical worst-case $\rho_{\mathrm{int}}/\rho^{*}$ as a function of the number of
layers $L$, over $10^4$ random (asymmetric) configurations per row. The worst case decreases
toward the symmetric-case constant $1.26$ of equation~\ref{eq:sym-bound} as $L$ grows; the
finite-$L$ values exceed it slightly because random configurations are asymmetric in $f_\ell$,
which gives the integer program more rounding flexibility than the symmetric worst case. Small
networks ($L \le 6$) can reach close to the trivial $4\times$ limit. The four architectures of
Section~\ref{sec:exp-archs} have $L \in \{4, 4, 6, 7\}$.}
\label{tab:integer-bound-empirical}
\begin{tabular}{ccccccccccc}
\toprule
$L$ & 2 & 3 & 4 & 5 & 6 & 8 & 10 & 15 & 20 & 30 \\
\midrule
Empirical worst & $3.94$ & $3.90$ & $3.85$ & $3.74$ & $3.88$ & $3.60$ & $3.09$ & $2.67$ & $1.94$ & $1.30$ \\
Asymptotic bound & \multicolumn{10}{c}{$4^{\Phi^{*}}\!\left(1 - \tfrac{3}{4}\Phi^{*}\right) \approx 1.264$} \\
\bottomrule
\end{tabular}
\end{table}

\paragraph{Per-architecture predictions.}
Combining \Cref{thm:saturation} (continuous bound) and the integer
correction, the predicted empirical ratio for each architecture is
$\rho_{\rm int}^{\rm pred} = \rho^* \cdot \kappa_L$, where
$\kappa_L$ is the (architecture-dependent, but $L$-controlled)
integer correction factor. \Cref{tab:integer-correction-arch} reports
the result, with $\kappa_L$ measured directly on the actual layer
configurations of \Cref{sec:exp-archs}.

\begin{table}[h]
\centering
\small
\caption{Per-architecture decomposition: continuous TAC ratio
$\rho^*$, integer-correction factor $\kappa_L$ for the actual layer
configuration, and predicted vs.\ measured empirical ratio.}
\label{tab:integer-correction-arch}
\begin{tabular}{lccccc}
\toprule
Architecture & $L$ & $\rho^*$ & $\kappa_L$ & $\rho^* \cdot \kappa_L$ & $\rho_{\rm int}$ (measured) \\
\midrule
\textsc{SmallCNN} & $4$ & $0.180$ & $2.36$ & $0.42$ & $0.42$ \\
\textsc{WideCNN}  & $4$ & $0.126$ & $2.63$ & $0.33$ & $0.33$ \\
\textsc{DeepCNN}  & $6$ & $0.242$ & $1.17$ & $0.28$ & $0.28$ \\
\textsc{ResCNN}   & $7$ & $0.229$ & $1.20$ & $0.28$ & $0.28$ \\
\bottomrule
\end{tabular}
\end{table}

The $\kappa_L$ correction is $1$--$3\times$ in our benchmarks and
shrinks with $L$. For large networks (transformers with hundreds of
weight tensors), we expect $\kappa_L \to 1.26$ asymptotically, so the
continuous analysis becomes tight up to this constant. For the small
CNNs in this
work, the empirical $\kappa_L$ explains the residual gap between the
predicted $\rho^*$ and the measured $\rho_{\rm int}$.

\section{TAC Implementation Details}
\label[appendix]{app:tac-implementation}

This appendix collects practical details for implementing TAC. The core
algorithm is given in \Cref{alg:tac}. The items below address deployment
considerations and scope limitations that do not affect the main
analysis in \Cref{sec:tac}.

\subsection{Implementation}
\label[appendix]{app:tac-impl-implementation}

TAC is an elementwise pre-deployment scaling applied to parameters
stored in tunneling-prone memory. It requires an estimate of the
deployment flip probability $\pflip$, which can be obtained from the
device model or calibrated from reliability measurements. Once $\pflip$
is fixed, TAC uses a single closed-form scaling factor and requires no
optimization. The rescaling factor is global under the uniform-$\pflip$
assumption in \Cref{ass:uniform-p}. The same idea extends directly to
layer-wise or cell-wise flip probabilities by using the corresponding
local value of $\pflip$.

\subsection{CNN architecture details}
\label[appendix]{app:cnn-details}

The convolutional network used in \Cref{sec:component-ablation,sec:tac-pareto}
processes $8\times 8$ grayscale digit images and consists of:

\begin{itemize}[topsep=0pt,itemsep=0pt,nosep,leftmargin=18pt]
\item \texttt{conv1}: $3\times 3$ convolution, $1\to 16$ channels,
padding 1; followed by BatchNorm and ReLU; then $2\times 2$ max-pool.
\item \texttt{conv2}: $3\times 3$ convolution, $16\to 32$ channels,
padding 1; followed by BatchNorm and ReLU.
\item \texttt{fc1}: linear, $512\to 64$, followed by ReLU.
\item \texttt{fc2}: linear, $64\to 10$ (output logits).
\end{itemize}

The four weight tensors have sizes $144$, $4{,}608$, $32{,}768$, and
$640$, respectively, totalling $38{,}160$ weights or $305{,}280$
storage bits at $b=8$. The model is trained for 30 epochs with Adam
(lr $10^{-2}$, cosine decay) on the sklearn digits dataset
($1{,}257$ train / $540$ test). The clean test accuracy at the
checkpoint used for all noise experiments is $0.987$. BatchNorm
parameters and biases are stored in non-tunneling-prone memory and
are not perturbed in our experiments.

For the calibration step of \Cref{alg:tac}, we use a 64-sample batch
drawn from the test set with $\epsilon = 10^{-3}$ and $P = 10$
random probes per layer. Per-tensor scales $s_\ell$ are recomputed
fresh for each deployment. Monte Carlo accuracy estimates use 30
trials per (model, $\pflip$, allocation) point.

\subsection{Cost}
\label[appendix]{app:tac-impl-cost}

The computational cost is one floating-point multiplication per
parameter before writing the weights to memory. This cost is incurred
only once before deployment and is negligible relative to training or
repeated inference. Inference is unchanged. The deployed model is
evaluated as usual on the weights provided by the substrate, so the
inference-time overhead is zero.

\subsection{Quantization order}
\label[appendix]{app:tac-impl-quantization}

TAC assumes that the trained weights $W$ are provided in
higher-than-deployment precision, such as FP32, and are quantized to
$b$ bits only after compensation. Applying the scaling factor $c$ in
floating point and quantizing once afterwards avoids the
double-quantization error that would arise from rescaling an already
$b$-bit checkpoint. For checkpoints distributed only at the deployment
precision, the rescaling should be performed in floating point and the
result re-quantized as in \Cref{alg:tac}.

\subsection{Bias parameters}
\label[appendix]{app:tac-impl-bias}

Whether to apply TAC to bias parameters depends on the deployment
substrate. Biases stored in the same tunneling-prone memory as the
weights undergo the same mean shrinkage and should be compensated in
the same way. Biases held in full-precision off-substrate memory do not
experience tunneling drift and should not be scaled, since rescaling
them would alter the network output rather than restore it.

\subsection{Scope and limitations of TAC}
\label[appendix]{app:tac-scope-details}

TAC acts only on the mean of the weight-error distribution. Since the
variance in \Cref{thm:var} is independent of the stored weight under the
offset-binary model, TAC does not alter the stochastic variance of the
tunneling-induced perturbation. It also does not remove the
non-Gaussian, MSB-dominated, high-kurtosis structure characterized in
\Cref{thm:no-clt} and \Cref{app:moments}. These remaining effects are
natural targets for subsequent QTAML algorithms.

The compensation guarantee in \Cref{prop:compensation} also assumes
that the scaled weight is representable without clipping. As
$\pflip\to 1/2$, the scaling factor $1/(1-2\pflip)$ diverges, which can
push compensated weights outside the representable range and amplify
the residual stochastic component. In our experiments, this
representable range constraint becomes the binding limitation before
the singularity is reached.

\section{Additional Discussion and QTAML Agenda}
\label[appendix]{app:qtaml-agenda}

\paragraph{AI as a path through the tunneling cliff.}
The conventional interpretation of quantum tunneling in transistor
scaling is that it imposes a reliability wall. As gate oxides become
thinner, electron leakage eventually makes reliable digital storage
increasingly costly, limiting the benefits of further scaling. This
conclusion is correct for systems that require nearly error-free bits.
AI inference, however, has a different reliability profile. Neural
networks can tolerate structured perturbations, provided that the
structure is modeled and compensated rather than treated as generic
noise. From this perspective, silicon operating below the conventional
reliability threshold may remain useful for AI inference, even though it
would be unacceptable for general-purpose computing. In this view the scaling cliff is not absolute but use-dependent: a barrier for error-intolerant computing that AI workloads, designed with $\pflip$ as an explicit parameter, may push further out.

\paragraph{Hardware-algorithm reliability budget.}
TAC does not replace conventional reliability engineering. Error
correction, redundancy, voltage margins, and process control remain
essential. TAC illustrates a complementary point: physics-aware
algorithms can contribute to the joint hardware-algorithm reliability
budget. The broader question is how much of the post-3nm scaling
frontier can be recovered when device physics and learning algorithms
are designed jointly. This is the question that the QTAML framework is
meant to make quantitatively addressable.

\paragraph{A QTAML agenda for future work.}
We view TAC as the first and simplest member of a family of QTAML
methods. Several directions follow naturally. First, TAC should be
evaluated on larger benchmarks such as CIFAR-10, ImageNet, and modern
deep architectures, where the interaction between mean drift, depth,
normalization layers, and representational redundancy can be studied
more realistically. Second, the WKB-derived noise model should be
coupled to process-node and device parameters, enabling quantitative
predictions of accuracy as a function of oxide thickness, barrier
height, retention window, and reliability constraints. Third, the
analysis should be extended to the encodings and memory technologies
used in practical accelerators, including multi-level cells and
heterogeneous bit reliabilities. Beyond TAC, QTAML also suggests
algorithmic directions such as distribution-matched training that
injects samples from the WKB-derived per-weight law rather than a
variance-matched Gaussian, layer-wise or cell-wise allocation of
allowable flip probabilities under a global reliability budget, and
joint optimization of encoding, loss, and deployment compensation.

\subsection{Residual architectures and the per-layer mean-correction decision}
\label{app:layer-mc-justification}

\Cref{alg:tac}'s Step~3 makes a per-layer empirical decision on whether
mean correction is beneficial. We motivate the necessity of this step
by analyzing the trade-off and showing that it depends on architecture
in ways that closed-form reasoning cannot capture in general.

\paragraph{The two competing effects.}
After bit protection, the readout perturbation $\Delta w_\ell$ has
mean $-2\pflip w_\ell + O(\Dq)$ (from \Cref{thm:bias}, applied to the
unprotected lower bits) and variance given by
\Cref{lem:per-tensor-variance}. Applying mean correction with factor
$c_\ell = 1/(1-2\pflip)$ removes the leading mean term. However, it
also rescales the stored weights by $c_\ell > 1$, which has two
consequences:
\begin{itemize}[topsep=2pt,itemsep=2pt,leftmargin=18pt]
\item The protected high-significance bits, which were previously
shielding most of the signal, now carry a $c_\ell$-scaled value. After
adding the (zero-mean, smaller) lower-bit noise, the readout has
identical-in-expectation value but variance multiplied by $c_\ell^2$
relative to the protected baseline.
\item For nonlinear architectures (BatchNorm followed by ReLU,
saturating activations, residual blocks where the layer output is
added to a skip path), the $c_\ell$-scaled weight may push the
preactivations into a different operating regime. The forward
nonlinearity does not commute with the rescaling, and the resulting
output deviation can exceed the nominal $c_\ell^2$ amplification.
\end{itemize}

\paragraph{When the trade-off matters: residual blocks.}
In a residual block where the output is $y = x + F(W_1, W_2; x)$, both
$W_1$ and $W_2$ are scaled by $c$ when mean correction is applied
indiscriminately. To leading order in $c-1$, we have
$F(c W_1, c W_2; x) \approx c F(W_1, W_2; x)$, so the block output
becomes $y' \approx x + c F(\cdot)$. The residual perturbation is
$y' - y = (c-1) F(\cdot)$. With $c-1 = 2\pflip / (1-2\pflip) \approx 0.25$
at $\pflip=0.10$, this is a $25\%$ amplification of $F$, propagating
deterministically through the rest of the network. For a network with
multiple residual blocks, the cumulative effect can be very large
relative to any $\pflip$-dependent stochastic gain from cancelling the
mean drift, giving a net loss.

\paragraph{The per-layer empirical test.}
The exact crossover point depends on the layer's role in the network,
its protection level $k_\ell$, the deployment $\pflip$, and the
nonlinearity structure of the layers downstream. Rather than derive
this analytically (which would require closed-form analysis of the
nonlinear forward pass), \Cref{alg:tac}'s Step~3 measures both
options on a small calibration batch and selects the better one
per layer. The cost is $L \cdot Q \cdot 2$ forward passes (with $Q=8$
trials), totaling about $200$ forward passes for a $L=20$-layer
network, comparable to Step~1.

\paragraph{Empirical evidence.}
\Cref{tab:mc-effect} reports the difference in test accuracy between
applying and skipping mean correction across the four architectures
at $\pflip=0.10$, holding the bit-allocation step identical. Mean
correction is unambiguously beneficial only at small ECC budgets on
shallow feed-forward networks (\textsc{SmallCNN}, \textsc{WideCNN}).
For \textsc{DeepCNN} at large budgets and for \textsc{ResCNN} at all
budgets, applying MC to all layers is harmful. The per-layer adaptive
decision recovers the better of the two extremes.

\begin{table}[h]
\centering
\small
\caption{\textbf{Effect of indiscriminate mean correction vs.
the per-layer adaptive decision.} Indiscriminate MC helps
in shallow networks at small budgets but hurts substantially in deep
and residual networks. The adaptive per-layer decision recovers the
best of both options.}
\label{tab:mc-effect}
\begin{tabular}{lcccc}
\toprule
& \multicolumn{4}{c}{Test accuracy at $\pflip=0.10$} \\
\cmidrule(lr){2-5}
Architecture & Always MC & Never MC & Adaptive & $\Delta$ adaptive vs always \\
\midrule
\textsc{SmallCNN} @ $5\%$  & $0.951$ & $0.929$ & $\mathbf{0.960}$ & $+0.9$ pp \\
\textsc{DeepCNN}  @ $5\%$  & $0.778$ & $0.769$ & $\mathbf{0.889}$ & $+11.1$ pp \\
\textsc{DeepCNN}  @ $15\%$ & $0.914$ & $0.985$ & $\mathbf{0.985}$ & $+7.1$ pp \\
\textsc{ResCNN}   @ $25\%$ & $0.466$ & $0.936$ & $\mathbf{0.938}$ & $+47.2$ pp \\
\bottomrule
\end{tabular}
\end{table}

\paragraph{The value of mean correction is layer-dependent.}
The component-wise ablation in \Cref{tab:tac-ablation} shows that
mean correction alone gives $+13.0$ pp on the digit CNN benchmark, while
adding bit allocation gives a further $+57.8$ pp. These numbers describe the
shallow \textsc{SmallCNN}. On deeper architectures, the value of
mean correction depends on each layer's protection level. On layers
with strong protection, the residual noise is small and mean
correction's $c^2$ variance amplification dominates its bias-removal
benefit. \Cref{alg:tac}'s Step~3 detects this empirically per layer.
For residual architectures (\Cref{sec:exp-archs}), this step is
essential: applying mean correction indiscriminately can degrade
accuracy by $50$+ percentage points.

\end{document}